%% file: main.tex
\documentclass[twoside]{article}

\usepackage{lipsum}
\usepackage{graphicx}
\usepackage{subfigure}

\usepackage[accepted]{aistats2025}

\usepackage[dvipsnames]{xcolor}
\definecolor{darkblue}{rgb}{0.0, 0.0, 0.55}
\definecolor{dark2blue}{rgb}{0.0, 0.0, 0.4}
\definecolor{darkred}{rgb}{0.55, 0.0, 0.0}
\definecolor{darkgreen}{rgb}{0, 0.5, 0.0}

\input{package}

\newcommand{\R}{\mathbb{R}}

\usepackage{tikz}
\usepackage{tikz-qtree}
\usetikzlibrary{shadows,trees}

\usepackage{hyperref}
\hypersetup{colorlinks=true,citecolor=darkblue,linkcolor=red}

\begin{document}



\runningauthor{K. Rojas, Y. Tan, M. Tao, Y. Nevmyvaka, and W. Deng}

\twocolumn[

\aistatstitle{Variational Schr\"odinger Momentum Diffusion}

\aistatsauthor{ 
    Kevin Rojas$^{* 1}$, $\ $ Yixin Tan$^{* 2}$, $\ $Molei Tao$^{1}$, $\ $Yuriy Nevmyvaka$^{3}$, $\ $Wei Deng$^{3}$ 
}

\aistatsaddress{$^{1}$Georgia Institute of Technology \And $^{2}$Duke University \And $^{3}$Morgan Stanley}
]
\footnotetext[1]{Equal contribution. K. Rojas conducted this work during his internship at Morgan Stanley.}
\footnotetext[2]{Correspondence: Wei Deng - weideng056@gmail.com.}

\input{chapters/abstract}

\input{chapters/introduction}

\input{chapters/related_works}

\input{chapters/preliminaries}

\input{chapters/method}

% % \input{../main_modules}
\input{chapters/exp_1_simulations}
\input{chapters/exp_2_time_series}
\input{chapters/exp_3_image_generation}

\input{chapters/conclusions}

%%%%%%%%%%%%%%%%%%%%%%%%%%%%%%%%%%%%%%%%%%%%%%%%%%%%%%%%%%%%
% \section*{Checklist}

% \newpage
\bibliography{mybib, other}
\bibliographystyle{apalike}

\newpage

\appendix

\newpage 
\onecolumn

\begin{large}
\begin{center}
    \textbf{Supplementary Material for \textbf{``Variational Schr\"odinger Momentum Diffusion''}}
\end{center}
\end{large}
$\newline$

\input{chapters/append_theory}

\input{chapters/append_images}

\input{chapters/append_time_series}

\end{document}

%% file: package.tex
\usepackage{comment}
\usepackage{bm}
\usepackage{url}            % simple URL typesetting
\usepackage{booktabs}       % professional-quality tables
\usepackage{amsfonts}       % blackboard math symbols
\usepackage{amsmath,amssymb}
\usepackage{empheq}
\usepackage{amsthm}
\usepackage{nicefrac}       % compact symbols for 1/2, etc.
\usepackage{microtype}      % microtypography
\usepackage{enumerate,natbib,mathrsfs}
\usepackage{tablefootnote}
\usepackage{wrapfig}
\usepackage{graphicx}
\usepackage{subfigure}
\usepackage{enumitem}
\usepackage{multirow}
\usepackage{extarrows}
\usepackage[OT1]{fontenc}
\usepackage[bookmarks=false]{hyperref}

\usepackage{amsmath}

\definecolor{darkblue}{rgb}{0.0, 0.0, 0.55}
\definecolor{darkblue}{rgb}{0.0, 0.0, 0.55}

\usepackage{hyperref}
\hypersetup{
  pdffitwindow=true,
  pdfstartview={FitH},
  pdfnewwindow=true,
  colorlinks,
  linktocpage=true,
  linkcolor=red,
  urlcolor=red,
  citecolor=darkblue
}

\usepackage{hyperref}
\hypersetup{colorlinks=true,citecolor=darkblue,linkcolor=red}
\usepackage{lipsum}

\DeclareMathOperator*{\argmin}{arg\,min}

\newcommand{\bI}{{\bm{\mathrm{I}}}}
\newcommand{\bJ}{{\bm{\mathrm{J}}}}

\newcommand{\bu}{{\bm{\mathrm{u}}}}

\newcommand{\ba}{{\bm{\mathrm{a}}}}
\newcommand{\rmR}{{{\mathrm{R}}}}
\newcommand{\bx}{{\boldsymbol x}}
\newcommand{\bv}{{\bm{\mathrm{v}}}}

\newcommand{\by}{{\boldsymbol y}}

\newcommand{\bw}{{\bm{\mathrm{w}}}} % \mathbf{w}_t
\newcommand{\bbf}{{\bm{\mathrm{f}}}}
\newcommand{\bbg}{{\bm{\mathrm{g}}}}

\newcommand{\bA}{{\mathrm{\mathbf{A}}}}
\newcommand{\bTheta}{{\mathrm{\mathbf{\Theta}}}}
\newcommand{\bB}{{\mathrm{\mathbf{B}}}}
\newcommand{\bC}{{\mathrm{\mathbf{C}}}}
\newcommand{\bD}{{\mathrm{\mathbf{D}}}}

\newcommand{\bH}{{\mathrm{\mathbf{H}}}}
\newcommand{\bL}{{\mathrm{\mathbf{L}}}}

\newcommand{\MX}{{\mathcal{X}}}

\newcommand{\bepsilon}{{\boldsymbol \epsilon}}
\newcommand{\bmu}{{\boldsymbol \mu}}

\newcommand{\bSigma}{{\boldsymbol \Sigma}}

\newcommand{\dd}{\mathcal{\dagger}}

\def\eqref#1{(\ref{#1})}

\newcommand{\ea}{\end{array}}
\newcommand{\ee}{\end{equation}}
\newcommand{\bea}{\begin{eqnarray}}
\newcommand{\eea}{\end{eqnarray}}
\newcommand{\beaa}{\begin{eqnarray*}}
\newcommand{\eeaa}{\end{eqnarray*}}

\def\E{\mathbb{E}}

\def\bD{{\bf D}}

\def\bx{{\bf x}}
\def\by{{\bf y}}
\def\bz{{\bf z}}

\newcommand{\basa}{\begin{assumption}}
\newcommand{\easa}{\end{assumption}}

\newcommand{\bas}{\begin{assum}}
\newcommand{\eas}{\end{assum}}

\def\dd{\mathrm{d}}

\def\limP2{\,\mathop{\buildrel \Pi_2\over\longrightarrow\,}}

\def\1{{\bf 1}}
\def\by{{\bf y}}

\def\:{\!:\!}

\newtheorem{assump}{Assumption}

\usepackage{algorithm}
\usepackage{algorithmic}

\newtheorem{remark}{Remark}
\newtheorem{theorem}{Theorem}

\newtheorem{lemma}{Lemma}
\newtheorem{proposition}{Proposition}
\newtheorem{assumption}{Assumption}

%% file: chapters/abstract.tex
\begin{abstract}
  The Momentum Schrödinger Bridge (mSB) \citep{mSB} has emerged as a leading method for accelerating generative diffusion processes and reducing transport costs. %\tao{what type of guarantees? i dont think our mmmSB paper had accuracy bound.} \Wei{let me try a different expression} 
  However, the lack of simulation-free properties inevitably results in high training costs and affects scalability. To obtain a trade-off between transport properties and scalability, we introduce variational Schr\"odinger momentum diffusion (VSMD), which employs linearized forward score functions (variational scores) to eliminate the dependence on simulated forward trajectories. Our approach leverages a multivariate diffusion process with adaptively transport-optimized variational scores. Additionally, we apply a critical-damping transform to stabilize training by removing the need for score estimations for both velocity and samples. Theoretically, we prove the convergence of samples generated with optimal variational scores and momentum diffusion. Empirical results demonstrate that VSMD efficiently generates anisotropic shapes while maintaining transport efficacy, outperforming overdamped alternatives, and avoiding complex denoising processes.  Our approach also scales effectively to real-world data, achieving competitive results in time series and image generation.%, both in unconditional and conditional settings.
\end{abstract}

% \tao{very nice work!} \Wei{Thanks!}

%% file: chapters/introduction.tex
\section{Introduction}

Score-based generative models (SGMs) have become the preferred method for generative modeling, showcasing exceptional capabilities in generating images, videos, and audios \citep{SGMS_beat_GAN, imagen_video, DiffWave, text_2_image}. To improve efficiency and simplify the denoising process, critically-damped Langevin diffusion (CLD) \citep{CLD} leverages kinetic (second-order) Langevin dynamics \citep{dalalyan_riou-durand_2020} %Hamiltonian dynamics \tao{`introduces momentum' or `leverages kinetic/2nd-order Langevin dynamics' instead of `leverages Hamiltonian dynamics'? in theory CLD is not Hamiltonian because there is additional noise and friction} \Wei{fixed} 
%to accelerate diffusion \tao{my personal take is, CLD does not necessarily accelerate, and accelerated convergence to invariant distribution may not matter too much for generative modeling. instead, one good thing about CLD is that score is no longer ill-defined at t=0 because it is the momentum gradient and density in the momentum direction is fully supported} 
by incorporating auxiliary velocity variables, resulting in well-behaved score functions at the boundary. While both SGMs and CLDs offer scalability benefits and simulation-free properties, they lack guaranteed optimal transport (OT) properties \citep{Lavenant_Santambrogio_22} and often involve costly evaluations to produce high-quality content \citep{DDPM, Progressive_distillation, DPMsolver}.

In contrast, the Momentum Schrödinger Bridge (mSB) \citep{Chen16, Pavon_CPAM_21, Caluya21, DSB, mSB} focuses on optimizing a stochastic control objective to achieve entropic optimal transport. The extension of forward-backward stochastic differential equations (FB-SDEs) \citep{forward_backward_SDE} with velocity variables not only accelerates the processes but also simplifies the denoising process and lowers tuning costs. However, training the intractable forward score functions for optimal transport relies heavily on simulated trajectories and often requires an additional pipeline using SGMs or CLDs for warm-up training to scale up to real-world data \citep{DSB, forward_backward_SDE}. This prompts a critical question: How can we efficiently train momentum diffusion models from scratch while maintaining effective transport?

To address these challenges, we propose the Variational Schrödinger Momentum Diffusion (VSMD) model. Inspired by \cite{VSDM}, we adopt locally linearized variational scores using variational inference to restore simulation-free properties for training backward scores. Additionally, we introduce a critical-damping transform to simplify and stabilize training by reducing the need to estimate two variational scores associated with both velocity and samples. Unlike the single-variate CLD model, VSMD functions as an adaptively transport-optimized multivariate diffusion \citep{multivariateDM}, facilitating efficient training, a simplified denoising process, and effective transport \citep{forward_backward_SDE}. Our contributions are highlighted in three key aspects and presented in Figure \ref{fig:vsmd_vsdm}: 

\begin{itemize}
    \item We introduce the Variational Schr\"odinger Momentum Diffusion (VSMD), an adaptive multivariate diffusion with simulation-free properties. We derive a tailored critical-damping rule to streamline training by avoiding the complexity of estimating additional variational scores. 
    \item Theoretically, we identify the convergence of the adaptively transport-optimized multivariate diffusion using techniques from stochastic approximation \citep{RobbinsM1951} and stochastic differential equations.
    \item VSMD surpasses its overdamped counterparts by leveraging momentum accelerations and avoiding complex denoising processes. It demonstrates strong performance in conditional and unconditional generations in both images and time series data, all while eliminating the need for warm-up initializations.
\end{itemize}

\begin{figure}
    \centering
    \includegraphics[width=1.05\linewidth]{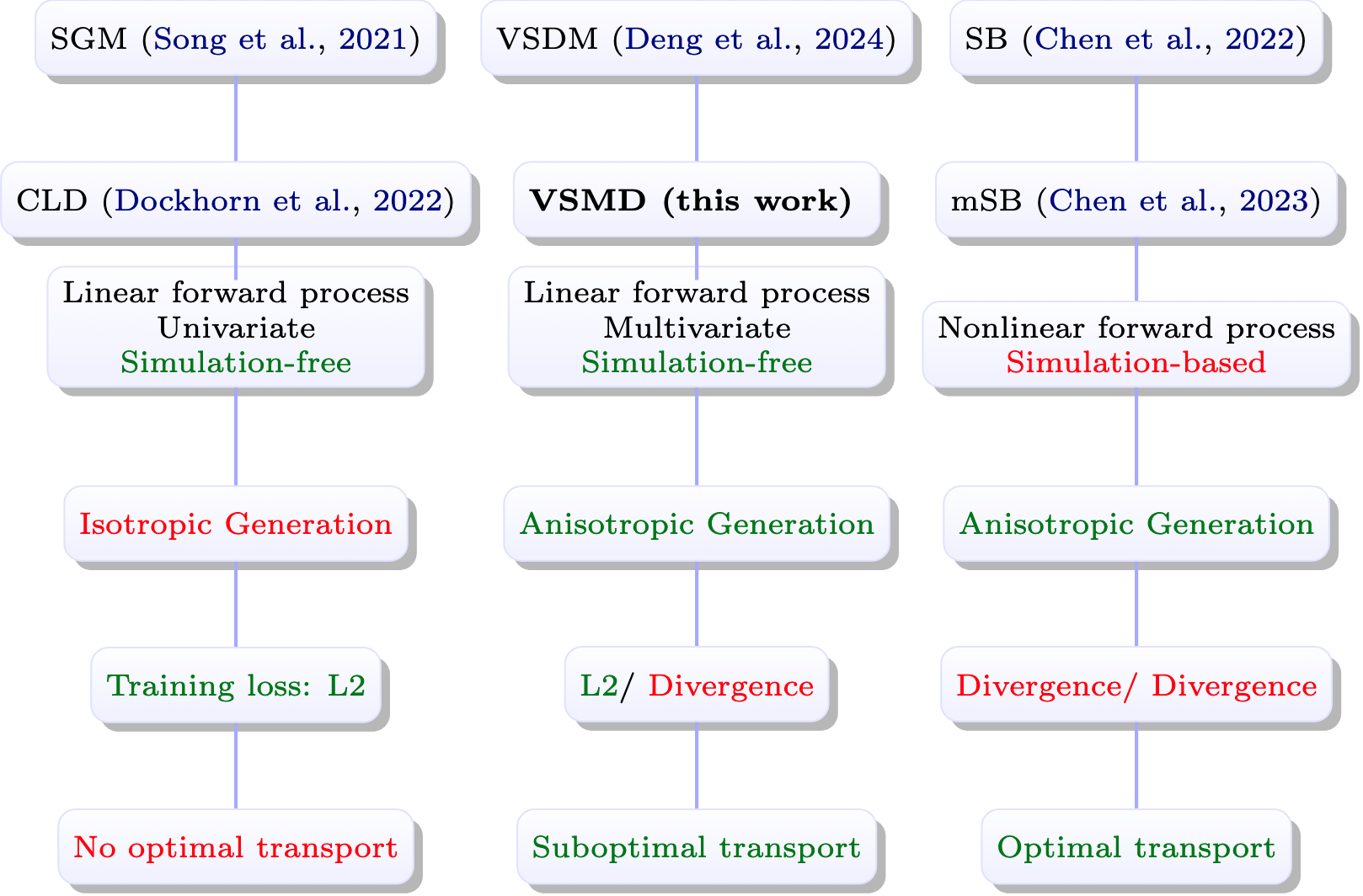}
    \vskip -0.05in
    \caption{Comparison with existing methodologies and algorithm properties.}
    \label{fig:vsmd_vsdm}
\end{figure}

%% file: chapters/related_works.tex
\section{Related Works}

\paragraph{Schrödinger Bridge (SB) Problems} Dynamic SB solvers for high-dimensional problems were initially introduced by \cite{DSB, forward_backward_SDE, SBP_max_llk, gefei_21, provably_schrodinger_bridge, rSB} to promote smoother trajectories with optimal transport properties. Subsequent work by \cite{SB_matching, Peluchetti23, SB_matching_momentum} enhanced performance by preserving marginal distributions and simplifying objectives inspired by bridge matching \citep{Rectified} and flow matching \citep{flow_matching}, which were further extended to general cost functions \citep{Wasserstein_Lagrangian_Flows, deep_generalized_SB}. To improve scalability, \cite{SB_flow_unpaired} proposed an online scheme to avoid caching samples and maintaining two networks. However, the need for simulations still limits scalability, highlighting the ongoing demand for more scalable methods.

\paragraph{Simulation-free Generative Models} \cite{flow_matching} proposed a simulation-free paradigm to train continuous normalizing flows \citep{neural_ode}, inherently connected to the OT displacement map \citep{McCann_97}. \cite{CFM_Tong, Multisample_flow_matching, Neural_Monge_Maps} advanced this field by leveraging minibatch OT objectives, approximations from discrete Sinkhorn solvers, non-independent couplings from minibatch data, and unbalanced Monge map estimators. \cite{Rectified, Rectified_group} introduced methods to rectify non-smooth trajectories and provide theoretical guarantees with convex cost functions. \cite{Albergo_stochastic_interpolants, Albergo_unified_framework} elegantly unified flow and diffusion models in a simulation-free manner. \cite{Aligned_SB} addressed data alignment issues, while \cite{neural_sb_adversarial, EOT_diffusion_process} employed adversarial objectives to optimize OT losses, although they still do not yield OT maps. \cite{light_SB} achieved simulation-free properties on small-scale problems by parameterizing the Schr\"odinger bridge potentials with Gaussian mixture distributions. \cite{NFDM} supported a broader family of forward diffusion and also introduced extra complexities. However, achieving simulation-free properties often requires sacrifices in OT properties, underscoring the need for more efficient schemes.

% \cite{provably_schrodinger_bridge};

%% file: chapters/preliminaries.tex
\section{Preliminaries}

\paragraph{SGMs:} Score-based generative models (SGMs) \citep{score_sde} have achieved unprecedented success in generative models. SGMs propose reversing a diffusion process to generate data distributions \citep{Anderson82}. However, the simplicity of the forward diffusion process, such as Brownian motion or the Ornstein-Uhlenbeck process, results in a complex denoising process that requires extensive tuning to generalize across different datasets.

\paragraph{CLD:} To address these challenges, critically-damped Langevin diffusion (CLD) has been proposed to accelerate diffusion by augmenting data $\bx_t\in \mathbb{R}^d$ with velocity variables $\bv_t\in\mathbb{R}^d$ motivated by Hamiltonian dynamics \citep{Neal12}:
\begin{align}\label{CLD_dyn}
   \begin{pmatrix}
    \dd \bx_t \\
    \dd \bv_t
   \end{pmatrix} &= 
   \frac{\beta}{2}
   \begin{pmatrix}
    \bv_t \\
    -\bx_t - \gamma \bv_t
   \end{pmatrix} \dd t + 
   \begin{pmatrix}
     \qquad\bm{0}_d \\
     \sqrt{\beta \gamma}\bI_d   
   \end{pmatrix}\dd \bw_t,            
\end{align}
where $\gamma$ is the friction coefficient that controls the randomness, $\mathbf{w}_t$ is the standard Brownian motion in $\mathbb{R}^{2d}$. In the long-time limit, the invariant distribution of continuous-time process \eqref{CLD_dyn} is a joint Gaussian distribution $\mathrm{N}(\bx; 0, \bI_d) \mathrm{N}(\bv; 0, \bI_d)$.

\paragraph{Damping Regimes:} The choices of $\gamma$ correspond to different damping regimes of Langevin dynamics \citep{classical_mechanics, CLD}. For high friction with $\gamma>2$, it leads to overdamped Langevin dynamics (LD) with straighter trajectories, however, the convergence speed is also impeded. In contrast, lower friction yields more oscillating trajectories and accelerates convergence. The dynamics with $\gamma=2$ are termed the critical-damped Langevin diffusion (CLD) while $\gamma<2$ corresponds to underdamped Langevin diffusion (ULD). Theoretically, CLD provides a balance between oscillation and speed, though in practice, different damping may need to be selected for optimal trade-off.

\paragraph{The Effect of Friction on Convergence:} The impact of the friction $\gamma$ on convergence speed is well understood. In particular, LD requires $\Omega(d/\epsilon^2)$ iterations to achieve an $\epsilon$ error in 2-Wasserstein (W2) distance for strongly log-concave distributions, whereas ULD requires only $\Omega(\sqrt{d}/\epsilon)$ iterations to achieve the same accuracy \citep{ULD_convergence}. Additional literature supporting the speed advantage of employing Hamiltonian dynamics can be found in \cite{Mangoubi18_leapfrog, dalalyan_riou-durand_2020, mixing_hmc}.

%% file: chapters/method.tex
\section{Variational Schr\"odinger Momentum Diffusion}

\subsection{Momentum Schr\"odinger Bridge}
The momentum Schr\"odinger bridge (mSB) \citep{Caluya21, mSB} can be interpreted as a stochastic optimal control (SOC) objective with optimal transport guarantees \citep{Chen21}:
\begin{align}
    &\inf_{\bu\in \mathcal{U}} \E\bigg\{\int_0^T \frac{1}{2}\|\bu(\overrightarrow\ba,t)\|^2_2 \mathrm{d}t \bigg\} \notag \\
    \text{s.t.} &\ \ \mathrm{d} \overrightarrow\ba_t=\left[\bbf(\overrightarrow\ba_t)+\bbg\bu(\overrightarrow\ba_t,t)\right]\mathrm{d}t+ \bbg\mathrm{d} \overrightarrow{\mathbf{w}_t} \label{control_diffusion} \\
    &\bbf(\overrightarrow\ba):= -\frac{\beta}{2}\bigg(\begin{pmatrix}
        0 & -1 \\
        1 & \gamma 
        \end{pmatrix} \otimes \bI_d\bigg) \overrightarrow\ba,\  \bbg:= \sqrt{\beta \gamma}\bJ_{2d}, \notag \\
    &\overrightarrow\ba_0\sim  \rho_{0}:=p_{\text{data}}\otimes p_{\bv_0} ,\ \overrightarrow\ba_T\sim  \rho_{T}:=p_{\text{prior}}\otimes p_{\bv_T} 
    , \notag
\end{align}
where $\bJ_{2d}=\begin{pmatrix}
0 & 0 \\
0 & 1
\end{pmatrix} \otimes \bI_d$, $\overrightarrow\ba=\begin{pmatrix}
    \overrightarrow\bx \\
    \overrightarrow\bv
   \end{pmatrix}\in\mathbb{R}^{2d}$ is the augmented variable;  $\bu:\mathbb{R}^{2d}\times [0, T]\rightarrow \mathbb{R}^{2d}$ is a control; $\bbf:=\mathbb{R}^{2d}\times [0, T]\rightarrow \mathbb{R}^{2d}$ is a vector field. The probability density function (PDF) for the process \eqref{control_diffusion} is denoted by $\overrightarrow\rho(\cdot, t)$. We fix $p_{\bv_0}$ and $p_{\bv_T}$ as the standard Gaussian distribution $\mathrm{N}(0, \bI)$.

The Lagrangian of Eq.\eqref{control_diffusion} leads to the Hamilton–Jacobi–Bellman (HJB) equation \citep{Caluya21, mSB}; applying the Hopf-Cole transform, we can solve the \emph{Schr\"{o}dinger system}  via the backward-forward Kolmogorov equations 
 \begin{align*}
    \footnotesize
    &\qquad \begin{cases}
    \frac{\partial \overrightarrow\psi}{\partial t}+\langle \nabla \overrightarrow\psi, \bbf \rangle +\frac{1}{2} \bbg \bbg^\intercal\Delta\overrightarrow\psi=0 \\[3pt]
    \frac{\partial \overleftarrow\varphi}{\partial t}+\nabla\cdot (\overleftarrow\varphi \bbf)-\frac{1}{2} \bbg \bbg^\intercal \Delta \overleftarrow\varphi=0,
    \end{cases}\\
    &\text{\ \ s.t. } \overrightarrow\psi(\bx, 0) \overleftarrow\varphi(\bx, 0) = \rho_0,\ \ ~\overrightarrow\psi(\by, T) \overleftarrow\varphi(\by, T) = \rho_T.
\end{align*}
    
Considering the stochastic representation for the forward Kolmogorov equation and the time reversal \citep{Anderson82}, we have the forward-backward stochastic differential equation (FB-SDE) \cite{mSB}:
\begin{subequations}
\begin{align}
\mathrm{d} \overrightarrow\ba_t&=\left[\bbf(\overrightarrow\ba_t, t) +  \bbg \bbg^\intercal \begin{pmatrix}
         \bm{0} \\
         \nabla_{\bv}\log\overrightarrow\psi(\overrightarrow\ba_t, t)
        \end{pmatrix} \right]\mathrm{d}t \notag \\
        &\qquad +\bbg \mathrm{d} \overrightarrow{\mathbf{w}_t}, \ \  \overrightarrow\ba_0\sim \rho_0, \label{f-sde}\\
\mathrm{d} \overleftarrow\ba_t&=\left[\bbf(\overleftarrow\ba_t, t) - \bbg \bbg^\intercal \begin{pmatrix}
         \bm{0} \\
         \nabla_{\bv}\log\overleftarrow\varphi(\overleftarrow\ba_t, t)
        \end{pmatrix}  \right]\mathrm{d}t \notag \\
        &\qquad+ \bbg \mathrm{d} \overleftarrow{\mathbf{w}}_t,\ \  \overleftarrow\ba_T \sim \rho_T. \label{b-sde}
\end{align}\label{FB-SDE}
\end{subequations}

%%%%%%%%%%%%%%%%%%%%%%%%%%%%%%%%%%%%%%%%%%%%%%%%%%%%%%%%%%%%%%%%%%%%%%%%%%%%%%%%%
% Denote $\overrightarrow y_t = \log \overrightarrow\psi(\ba_t, t)$ and $\overleftarrow y_t=\log \overleftarrow\varphi(\ba_t, t)$.  ${\overrightarrow\bz_t =\sqrt{\beta} \nabla_{\bv} \overrightarrow y_t}$, ${\overleftarrow \bz_t =\sqrt{\beta} \nabla_{\bv} \overleftarrow y_t}$. 

We can next solve $(\overrightarrow\psi, \overleftarrow\varphi)$ for the augmented variable $\ba=\begin{pmatrix}
    \bx \\
    \bv
   \end{pmatrix}$ to the \emph{Schr\"{o}dinger system} by the nonlinear Feynman-Kac formula \citep{Ma_FB_SDE, forward_backward_SDE, mSB}:
\begin{proposition}[Feynman-Kac formula] 
\label{non_linear_feynman_kac}
Given $\beta, \gamma>0$, the stochastic representation of the solution follows 
% \footnote{$\overleftarrow y_T$ is often intractable and omitted in practical training.}
\begin{align}
    \overleftarrow y_s&=\E\bigg[\overleftarrow y_T -\int_s^T {\Gamma_{\zeta}(\overleftarrow\bz_t; \overrightarrow\bz_t)}\dd t \bigg|\overrightarrow\bx_s=\textbf{x}_s\bigg],\notag\\
     \Gamma_{\zeta}(\overleftarrow\bz_t; \overrightarrow\bz_t)& \footnotesize{\equiv\frac{1}{2} \|  \overleftarrow \bz_t\|_2^2  + \nabla_{\bv} \cdot \big(  \sqrt{\beta}\overleftarrow\bz_t - \bbf_t \big) + \zeta   \overleftarrow \bz_t^\intercal \overrightarrow \bz_t}\label{Gamma_def},
\end{align}
where $\overrightarrow y_t = \log \overrightarrow\psi(\ba_t, t)$ and $\overleftarrow y_t=\log \overleftarrow\varphi(\ba_t, t)$, ${\overrightarrow\bz_t =\sqrt{\beta} \nabla_{\bv} \overrightarrow y_t}$, ${\overleftarrow \bz_t =\sqrt{\beta} \nabla_{\bv} \overleftarrow y_t}$, and $\zeta=1$.
\end{proposition}

\subsection{Linear Approximation via Multivariate Momentum Diffusion}

Consider a linear approximation of the forward process \eqref{f-sde} with a fixed matrix $\bA_{\ba,t}=\begin{pmatrix}
         \bm{0}_d & \bm{0}_d \\
         \bA_{\bx,t} & \bA_{\bv,t}
        \end{pmatrix}\in \mathbb{R}^{2d\times 2d}$ (referred to as the variational score):
%\begin{subequations}
\begin{align}
\mathrm{d} \overrightarrow\ba_t&=\left[\bbf(\overrightarrow\ba_t, t) +  \bbg \bbg^\intercal \bA_{\ba,t} \overrightarrow\ba_t \right]\mathrm{d}t+\bbg \mathrm{d} \overrightarrow{\mathbf{w}_t} \notag \\
        &=-\frac{1}{2}\bD_t\beta \overrightarrow\ba_t \dd t + \bbg\dd \overrightarrow\bw_t \label{FB-SDE-linear-unified}\\
        \bD_t &= \begin{pmatrix}
        0 & -1 \\
        1 & \gamma 
        \end{pmatrix}\otimes \bI_d-2\gamma \bA_{\ba,t}, \notag
\end{align} %
%\end{subequations}
where $\bI - 2 \gamma \bA_{\bx,t}$ and $\bI - 2 \bA_{\bv,t} $ are restricted to a positive-definite matrix.

The mean and covariance of the augmented linear SDE \eqref{FB-SDE-linear-unified} follow that \citep{applied_sde}
\begin{subequations}
\begin{align}
    &\frac{\dd \bmu_{t|0}}{\dd t}=-\frac{1}{2}\beta \bD_t \bmu_{t|0}\label{mu_diffusion}\\
    &\frac{\dd \bSigma_{t|0}}{\dd t}=-\frac{1}{2}\beta \big(\bD_t  \bSigma_{t|0} +\bSigma_{t|0}\bD_t^{\intercal}\big)  + {\beta\gamma}\bJ_{2d} \label{sigma_diffusion},
\end{align}
\end{subequations}
where $\bD_t$ and $\bJ_{2d}$ are defined in Eq.\eqref{FB-SDE-linear-unified} and \eqref{control_diffusion}, respectively. Solving the mean process leads to the solution:
\begin{align}
    \bmu_{t|0}&=e^{-\frac{1}{2} \beta[\bD]_t} \bx_0,\label{mean_dyn}
\end{align}
where $[\bD]_t=\int_0^t \bD_s\dd s$. The covariance process is a differential Lyapunov matrix equation \citep{applied_sde} and can be solved by decomposing $\bSigma_{t|0}$ as $\bC_t \bH_t^{-1}$, where $\bC_t$ and $\bH_t$ follow that:
\begin{align}\label{cov_dynamics}
  \begin{pmatrix}
    \bC_t \\
    \bH_t
  \end{pmatrix} &=
  \exp\Bigg[
  \begin{pmatrix}
    -\frac{1}{2}\beta[\bD]_t & \gamma \beta\big[ \bm{\mathrm{J}}_{2d}\big]_t \\
    \bm{0} &  \frac{1}{2} \beta[\bD^{\intercal}]_t
  \end{pmatrix}
  \Bigg]
        \begin{pmatrix}
          {\bSigma}_0 \\
          {\bI_{2d}} 
        \end{pmatrix}.            
\end{align}
Additional speed-ups can be achieved on real-world datasets by avoiding the matrix exponential through the use of a time-invariant and diagonal $\bD$, as detailed in Appendix A of \cite{VSDM}.

Next, we can achieve the simulation-free update of the multivariate momentum diffusion as follows 
\begin{align}
    \overrightarrow\ba_t = \bmu_{t|0} + \bL_t\bepsilon,\label{update_as}
\end{align}
where $\bmu_{t|0}\sim\eqref{mean_dyn}$, $\bL_t$ is a lower-triangular matrix that satisfies $\bL_t \bL_t^{\intercal}=\bSigma_{t|0}$, and $\bepsilon \in\mathbb{R}^{2d}$ is a Gaussian vector. The forward PDF follows that
\begin{align*}
    \overrightarrow{\rho}_{t|0}(\overrightarrow\ba_t) &\propto \exp\bigg\{ -\frac{1}{2}(\overrightarrow\ba_t - \bmu_{t|0})^{\intercal} \bSigma_{t|0}^{-1} (\overrightarrow\ba_t - \bmu_{t|0}) \bigg\},
\end{align*}
which leads to a score function as follows
\begin{align}
     \nabla \log  \overrightarrow{\rho}_{t|0}(\overrightarrow\ba_t) &= -\bSigma_{t|0}^{-1}(\overrightarrow\ba_t-\bmu_t)=-\bL_t^{-\intercal}\bepsilon.\label{score_expression}
\end{align}
We next resort to a neural network parametrization $s_{t}(\cdot)$ via the following loss function to learn the score:
\begin{equation}\label{MDM_loss}
    \nabla_{\theta} \|-\bL_t^{-\intercal}\bm{\epsilon}-s_{t}(\overrightarrow\ba_t)\|_2^2.
\end{equation}

\subsubsection{Backward SDE}\label{sec:backward_sde}
The backward process follows by taking the time reverse \cite{Anderson82} of the forward process \eqref{FB-SDE-linear-unified}:
\begin{align}
    \dd \overleftarrow\ba = -\frac{1}{2}\bD_t\beta\overleftarrow\ba_t \dd t - \bbg \bbg^\intercal s_t(\overleftarrow\ba)\dd t + \bbg\dd \overleftarrow\bw_t,\label{backward_process}
\end{align}
where the prior distribution is restricted to a Gaussian distribution following $\ba_T \sim \mathrm{N}(\bm{0}, \bSigma_{T|0})$ as in Eq.\eqref{sigma_diffusion}.

\subsection{Adaptively Transport-Optimized Diffusion}

Among the infinite transportation plans, we aim to obtain the optimal variational scores $\bA^{\star}_{\ba,t}$ to ensure efficient transport. For that end, we resort to the SOC objective under a linearized SDE constraint:
\begin{align}
    &\inf_{\bA_{\bx}, \bA_{\bv}\in \mathbb{R}^{d\times d}} \E\bigg\{\int_0^T \frac{1}{2}\bigg\|\bA_{\ba,t} \overrightarrow\ba_t\bigg\|^2_2 \mathrm{d}t \bigg\} \notag \\
    \text{s.t.} &\ \ \mathrm{d} \overrightarrow\ba_t=\left[\bbf(\overrightarrow\ba_t, t) +  \bbg \bbg^\intercal \bA_{\ba,t} \overrightarrow\ba_t \right]\mathrm{d}t+\bbg \mathrm{d} \overrightarrow{\mathbf{w}_t}. \notag \\
        &\overrightarrow\ba_0\sim  \rho_{0}:=p_{\text{data}}\otimes p_{\bv_0} ,\ \overrightarrow\ba_T\sim  \rho_{T}:=p_{\text{prior}}\otimes p_{\bv_T}. \notag
\end{align}

Since the diffusion from $\rho_0$ to $\rho_T$ is nonlinear in general, a closed-form solution is often intractable. \cite{SB_closed_form} studied the analytic solution of Gaussian SB based on a Langevin prior process, however, the ULD prior process is still not well studied.

To tackle this issue, we first build a loss function through the Feynman-Kac formula in Proposition \ref{non_linear_feynman_kac}:
% \begin{subequations}
\begin{align}
    \overrightarrow{\mathcal{L}}(\bA)&=\small{-\int_0^T \E_{\overleftarrow\bx_t\backsim \eqref{backward_process}}\bigg[\Gamma_{\zeta}(\bA_{\ba,t}\ba_t; \overleftarrow\bz^{\theta}_t)\dd t \bigg|\overleftarrow\ba_T\bigg]}\label{SB-loss-f},
\end{align}
where $\overleftarrow\bz^{\theta}_t$ is the approximation of $\overleftarrow\bz_t$ in Eq.\eqref{Gamma_def}.

We next employ stochastic approximation (SA) \citep{RobbinsM1951, Albert90} to optimize the variational score $\bA_{\ba,t}$ adaptively for achieving more efficient transportation plans. 
\begin{itemize}
\item[(1)] Sample $\{\overleftarrow\bx^{(k+1)}_{t_i}\}_{i=0}^{N-1}$ via the backward SDE \eqref{backward_process};
\item[(2)] Stochastic approximation of $\big\{\bA^{(k)}_{\ba, t_i}\}_{i=0}^{N-1}$: 
$$\bA_{\ba, t_i}^{(k+1)}=\bA_{\ba, t_i}^{(k)}-\eta_{k+1} \nabla\overrightarrow{\mathcal{L}}_{t_i}(\bA_{\ba, t_i}^{(k)}; \overleftarrow\bx^{(k+1)}_{t_i}),$$
\end{itemize}
where $\eta_{k+1}$ is the step size, $\{t_0, t_1, \cdots, t_{N-1}\}$ is a collection of time discretization through the Euler–Maruyama (EM) or symmetric splitting scheme \citep{CLD}, $\nabla\overrightarrow{\mathcal{L}}_{t_i}(\bA_{\ba, t_i}^{(k)}; \overleftarrow\bx^{(k+1)}_{t_i})$ is the stochastic gradient of Eq.\eqref{SB-loss-f} at time $t_i$ given $\overleftarrow\bx^{(k+1)}_{t_i}$. 

We expect that under mild assumptions, $\{\bA^{(k)}_{\ba,t}\}_t$ will converge to a local optimum $\{\bA^{\star}_{\ba,t}\}_t$ that yields sub-optimal transport properties and the score function $\{s^{\theta_{\star}}_t\}$ given $\{\bA^{\star}_{\ba,t}\}_t$ will be more effective to handle complex transport problems compared to the vanilla $\bA^{(k)}_{\ba,t}\equiv \bm{0}$ in CLD.

\paragraph{Connections to Half-bridge Solvers} 

mSB proposes to solve a general half-bridge (left) associated with the forward SDE \eqref{control_diffusion} for optimal transport. For scalability, the linear approximation in Eq.\eqref{FB-SDE-linear-unified} has limited the solution space into a class of generalized Ornstein-Uhlenbeck (gOU) processes (right):
\begin{align*}
    \argmin_{\mathbb{P}\in \mathcal{D}(\rho_{\text{data}},\  \cdot)} \text{KL}(\mathbb{P}\|\mathbb{P}_{2k-1}) \rightarrow \argmin_{\mathbb{\widehat P}\in \text{gOU}(\rho_{\text{data}},\  \cdot)} \text{KL}(\mathbb{\widehat P}\|\mathbb{P}_{2k-1})
\end{align*}
where $\mathcal{D}(\rho_{\text{data}}, \cdot)$ and $\text{gOU}(\rho_{\text{data}},\  \cdot)$ denote the classes of path couplings from $t=0$ to $T$ and the initial marginal follows $\rho_{\text{data}}$. The solution $\mathbb{\widehat P}$ acts as a local optimum of the optimal transport solution.

\subsection{Stabilization via Damping Transform}

We rewrite the forward process \eqref{FB-SDE-linear-unified} as a coupled probability flow ODE \citep{score_sde}
\begin{align*}
\dd\bx_t &= \frac{1}{2} \beta \bv_t \dd t \\
\dd\bv_t &= -\bigg[\bar\gamma \bv_t+\frac{2}{\beta}\bar\omega_0^2 \bx_t + \frac{1}{2} \beta\gamma \nabla_{\bv} \log  \overrightarrow{\rho}_{t|0}(\overrightarrow\ba_t)\bigg]\dd t.\notag
\end{align*}
where $\bar\gamma=\frac{1}{2} \beta (\gamma - 2\gamma \bA_{\bv,t})$, $\bar\omega_0^2=\frac{1}{4} \beta^2 (1-2 \gamma \bA_{\bx,t})$.

Regarding the balance between the mass oscillation and damping \citep{classical_mechanics}, we rewrite the coupled equations into a second-order differential equation:
\begin{align*}
&\frac{\dd^2 \bx_t}{\dd t^2}+  \bar\gamma \frac{\dd \bx_t}{\dd t} +\bar\omega_0^2 \bx_t +\frac{1}{2} \beta\gamma \nabla_{\bv}\log  \overrightarrow{\rho}_{t|0}(\overrightarrow\ba_t)=0.
\end{align*}

Applying the case of critical damping \citep{classical_mechanics}, we have that
\begin{align}
&\bar\gamma^2 =4 \rmR \bar\omega_0^2,\notag
\end{align}
where $\rmR \in (0, 1]$ is a scalar. The trade-off between oscillation and damping w.r.t. different $\rmR$ leads to two algorithms \citep{classical_mechanics, CLD}: 
\begin{itemize}
    \item $\rmR=1$ corresponds to critical damping (VSCLD);
    \item $\rmR<1$ leads to under-damping (VSULD).
\end{itemize}

After some transformations, we have that
\begin{align}
\bA_{\bv,t} = \frac{1}{2} - \frac{1}{\gamma}\sqrt{\rmR(1-2 \gamma \bA_{\bx,t})}.\label{damp_transform}
\end{align}

The above equation indicates that instead of training two modules $\bA_{\bv,t}$ and $\bA_{\bx,t}$, we can solely train one module such as  $\bA_{\bv,t}$ and apply the transformation \eqref{damp_transform} to infer the other. Such a transformation has greatly stabilized the training and alleviated the training cost.

As observed in \cite{CLD}, under-damping often yields fast mixing while compromising the smoothness of the trajectory. Empirically, we observe that under-damping can be much faster than critical damping and may only slightly decrease the straightness of the trajectories, which motivates us to tune $\rmR$ to obtain the best trade-off. Now we present our algorithm in Algorithm \ref{VSMD_alg}.

\begin{algorithm*}[!tb]
   \caption{Variational Schr\"odinger Momentum Diffusion (VSMD). The variational scores $\bA^{(0)}_{\ba}$ are initialized to $\bm{0}$ by default. Specify the diffusion hyperparameters $\beta, \gamma$. The damping ratios $\rmR=1$ and $\rmR<1$ correspond to the VSCLD and VSULD algorithms, respectively, balancing oscillation and damping. Given adaptively optimized $s_t^{(k+1)}$, $\overleftarrow\ba_{0}$ can be generated through the backward SDE \eqref{backward_process}. The continuous dynamics can be empirically discretized through the EM or symmetric splitting scheme.}
   \label{VSMD_alg}
\begin{algorithmic}
\REPEAT
   \STATE{\textbf{Optimization of the Score Function $s_t$ via Cached Dynamics}}
   \STATE{\text{Draw $\ba_0\sim p_{\text{data}}\otimes \mathrm{N}(\bm{0}, \bI)$}, compute the mean process $\bmu_{t|0}$ and $\begin{pmatrix}
    \bC_t \\
    \bH_t
  \end{pmatrix}$ by Eq.\eqref{mean_dyn} and \eqref{cov_dynamics}, respectively.} 
  \STATE{Compute the covariance $\bSigma_{t}=\bC_t \bH_t^{-1}$ and the Cholesky factor $\bL_{t}^{-\intercal}$, where $\bL_t \bL_t^{\intercal}=\bSigma_{t|0}$. Store $\bmu_{t|0}$, $\bSigma_{t|0}$, and $\bL_{t}^{-\intercal}$ in cache to speedup calculations.}
    \STATE{\text{Draw $\ba_t|\ba_0\sim \mathrm{N}(\bmu_{t|0}, \bSigma_{t|0})$ and $\bm{\epsilon}\sim\mathrm{N}(\bm{0}, \bI)$}.  \text{Optimize loss function to learn the score} $s^{(k+1)}_t$:
\begin{equation*} 
    \nabla_{\theta} \|-\bL_t^{-\intercal}\bm{\epsilon}-s^{(k+1)}_{t}(\overrightarrow\ba_t)\|_2^2.
\end{equation*}}

   \STATE{\textbf{Stochastic Approximation of Variational Scores $\bA_{\ba,t}$}}
   \STATE{\text{Simulate $\overleftarrow\bx^{(k+1)}_t$ via Eq.\eqref{backward_process} and optimize $\bA_{\bx, t}^{(k+1)}$ through the updates:}}
   \begin{equation*}
       \bA_{\bx, t}^{(k+1)}=\bA_{\bx, t}^{(k)}-\eta_{k+1} \nabla_{\bA_{\bx}}\overrightarrow{\mathcal{L}}_{t}(\bA_{\ba, t}^{(k)}; \overleftarrow\bx^{(k+1)}_{t}).
   \end{equation*}
   \STATE{Compute the damping transform $\bA^{(k+1)}_{\bv,t} = \frac{1}{2} - \frac{1}{\gamma}\sqrt{\rmR(1-2 \gamma \bA^{(k+1)}_{\bx,t})}$.}
   \UNTIL{The accuracy meets the criteria.}
    \vskip -3 in
\end{algorithmic}
\end{algorithm*}

% \subsection{Theoretical Analysis} % leave it in the appendix

%% file: chapters/exp_1_simulations.tex
% \newpage 

\section{Empirical Studies}
\subsection{Simulations}

We investigate anisotropic generation using two datasets: spiral and checkerboard. Specifically, we stretch the Y-axis of the spiral dataset by a factor of 8 and the X-axis of the checkerboard dataset by a factor of 6, referring to these modified datasets as spiral-8Y and checkerboard-6X, respectively.

\paragraph{Anisotropic Generation} % and Transport Efficiency}

We analyze CLD, ULD, VSCLD, and VSULD with various $\beta$ values, denoting them as CLD-$\beta$, ULD-$\beta$, VSCLD-$\beta$, and VSULD-$\beta$. The root mean square error (RMSE) of the probability mass functions (PMFs) between the generated samples and ground-truth samples is measured to assess performance.

Initially, we experiment with CLD-5 and observe that it fails to generate content effectively in the stretched dimension, as shown in Figure \ref{fig:VSULD_vs_VSDM}. In contrast, our VSULD model, with a damping ratio of 0.7, utilizes a faster speed for the stretched dimension and a slower speed for the non-stretched dimension, accurately addressing anisotropic generation.

\begin{figure}[!ht]
  \centering
  \vspace{-0.05in}
    \subfigure{\includegraphics[scale=0.125]{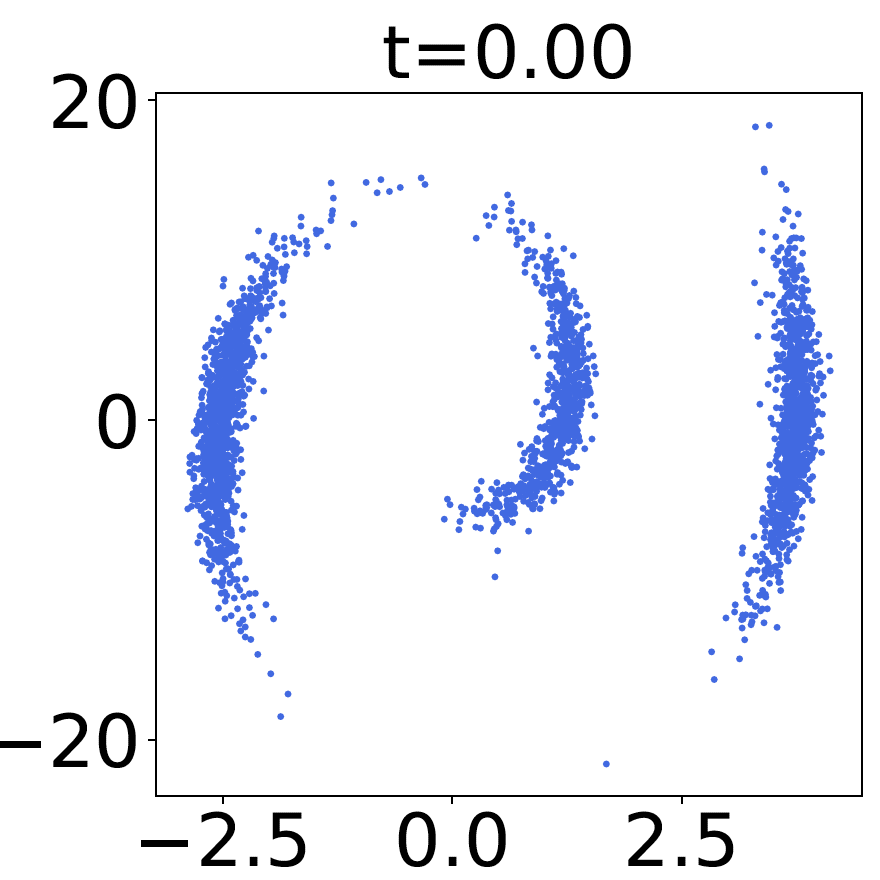}}
    \subfigure{\includegraphics[scale=0.125]{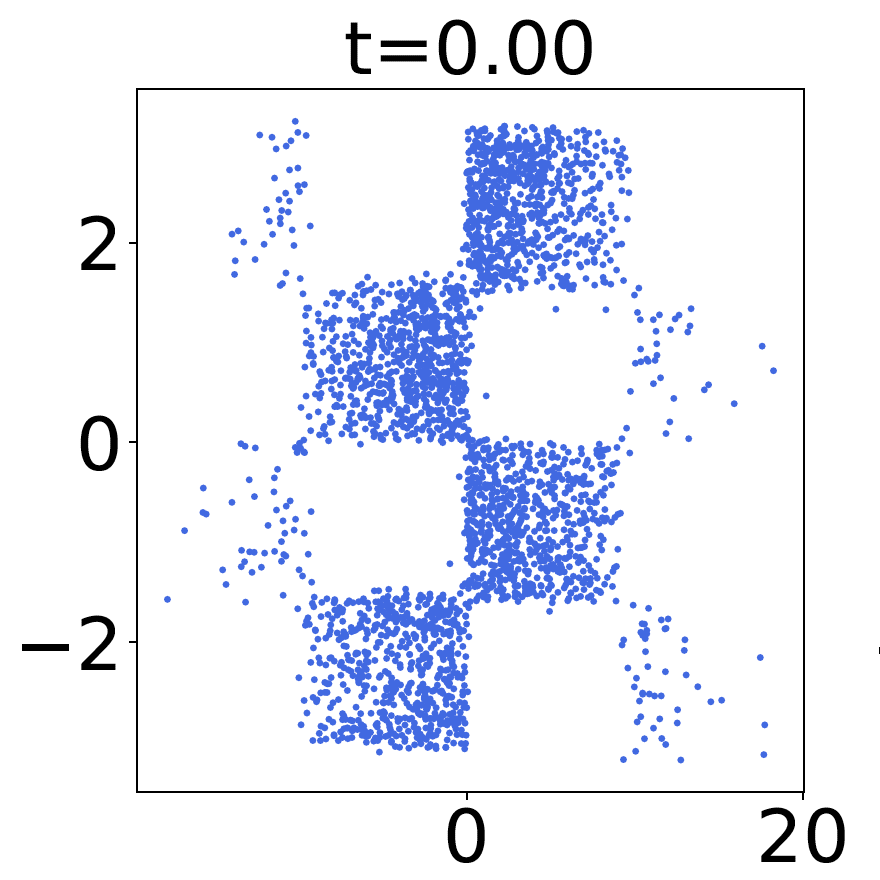}}
    \subfigure{\includegraphics[scale=0.125]{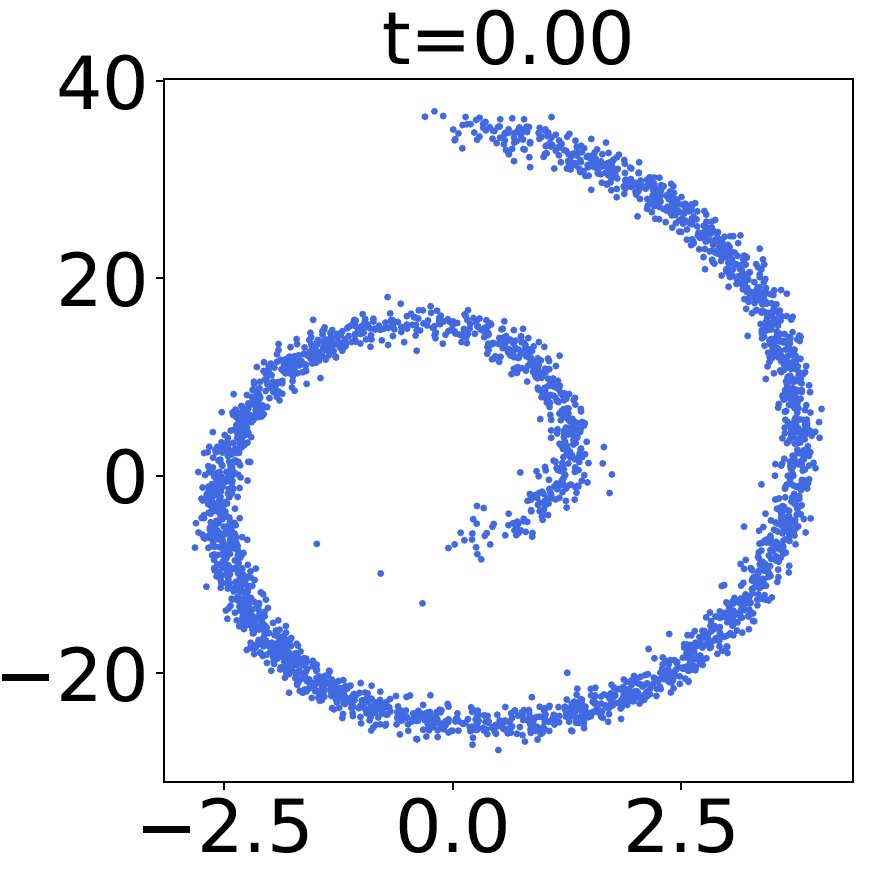}}
    \subfigure{\includegraphics[scale=0.125]{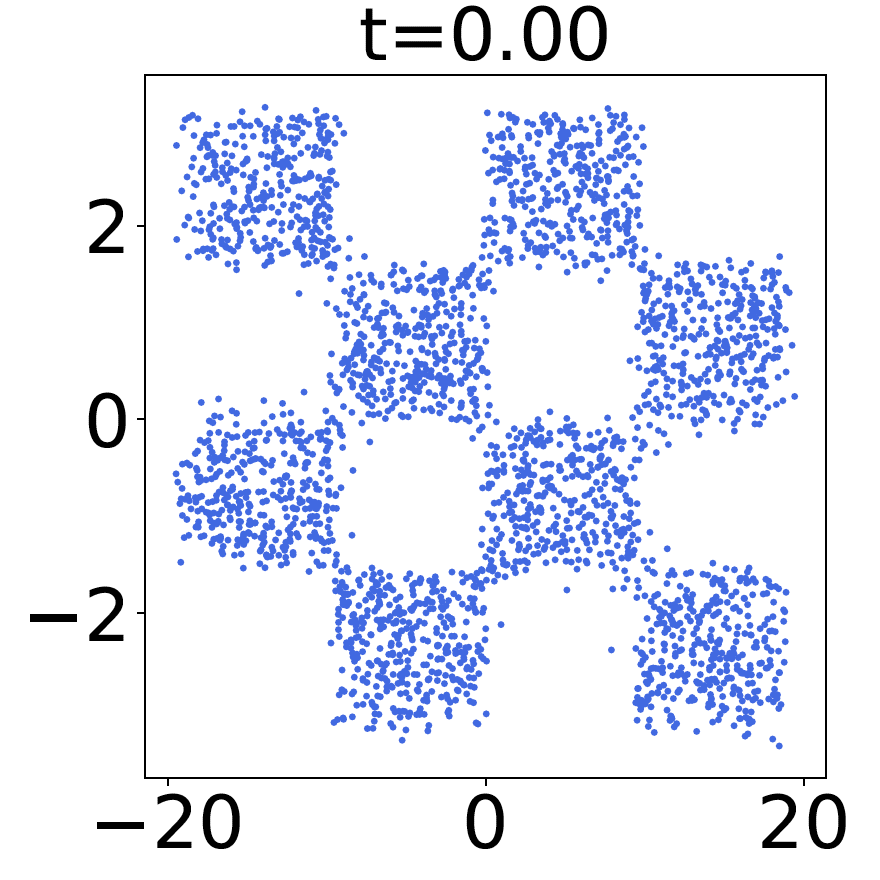}}
    \vspace{-0.1in}
  \caption{CLD-5 (left two) v.s. VSULD-5 (right two) on spiral-8Y and checkerboard-6X.}\label{fig:VSULD_vs_VSDM}
  \vspace{-1em}
\end{figure}

\paragraph{Trade-off between Sample Quality and Transport Efficiency} To improve the anisotropic generation of CLD, we increase $\beta$ and observe in Figure \ref{fig:generation_quality} that CLD-10 and ULD-10 exhibit comparable generation quality to VSULD-5. Additionally, we find that underdamped models such as ULD-5 and VSULD-5 converge faster than the critically-damped counterparts like CLD-5 and VSCLD-5, yielding slightly better sample quality.

\begin{figure}[!ht]
  % \centering
  \vspace{-0.07in}
  \subfigure[\small{Spiral-8Y}]{\includegraphics[scale=0.26]{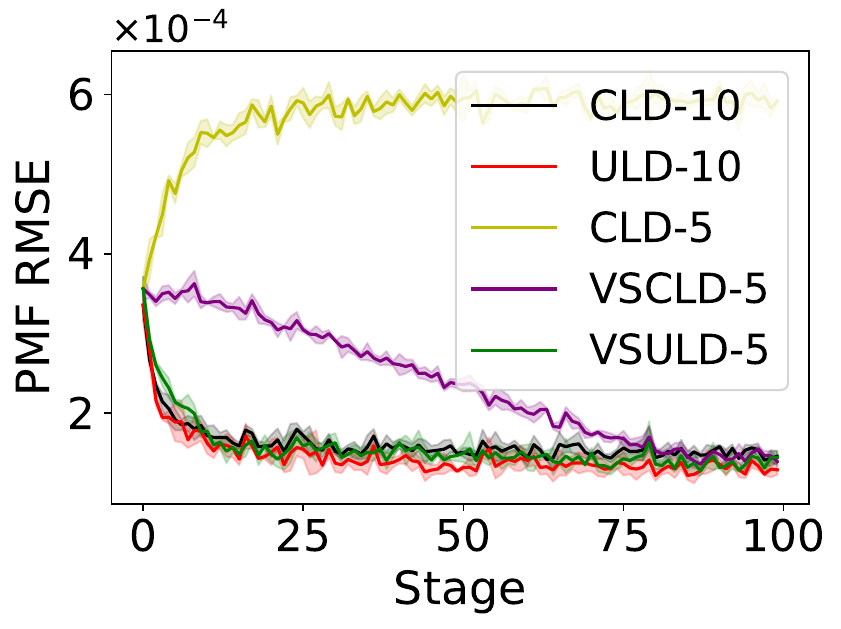}}\quad
  \subfigure[\small{Checkerboard-6X}]{\includegraphics[scale=0.26]{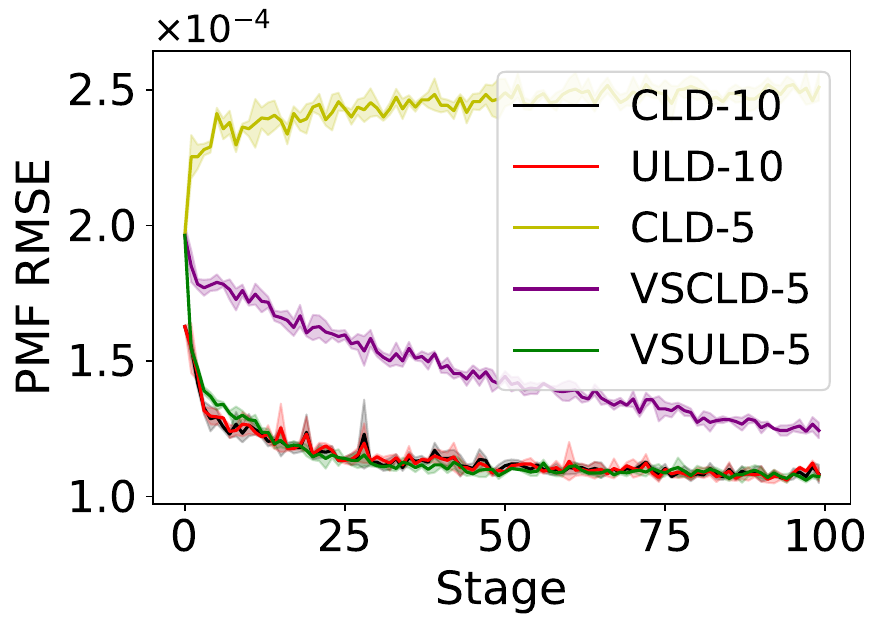}}
  \vskip -0.15in
  \caption{Sample quality evaluation. The damping ratios for ULD and VSULD are both fixed to 0.7.}\label{fig:generation_quality}
\end{figure}

Increasing $\beta$ significantly enhances anisotropic generation for CLD and ULD. However, a large $\beta$ results in inefficient transport for the non-stretched dimension (e.g., the X-axis of the spiral dataset). Specifically, evaluating the straightness metric as suggested in \cite{VSDM}, we observe in Figure \ref{Fig:trajectories_dynamics} that both ULD-10 and CLD-10 show significantly worse straightness compared to models with $\beta=5$, such as CLD-5, VSCLD-5, and VSULD-5. Furthermore, critically-damped models demonstrate marginally better straightness metrics than underdamped models, indicating a trade-off between convergence speed and transport efficiency.

\begin{figure}[!ht]
  % \centering
  \vspace{-0.07in}
  \subfigure[\small{Spiral-8Y}]{\includegraphics[scale=0.26]{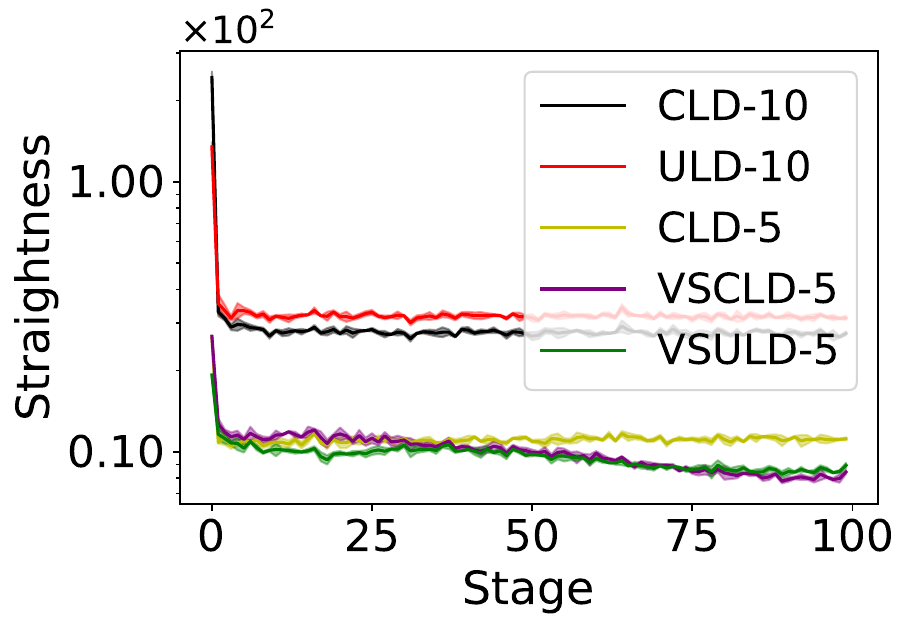}} 
  \subfigure[\small{Checkerboard-6X}]{\includegraphics[scale=0.26]{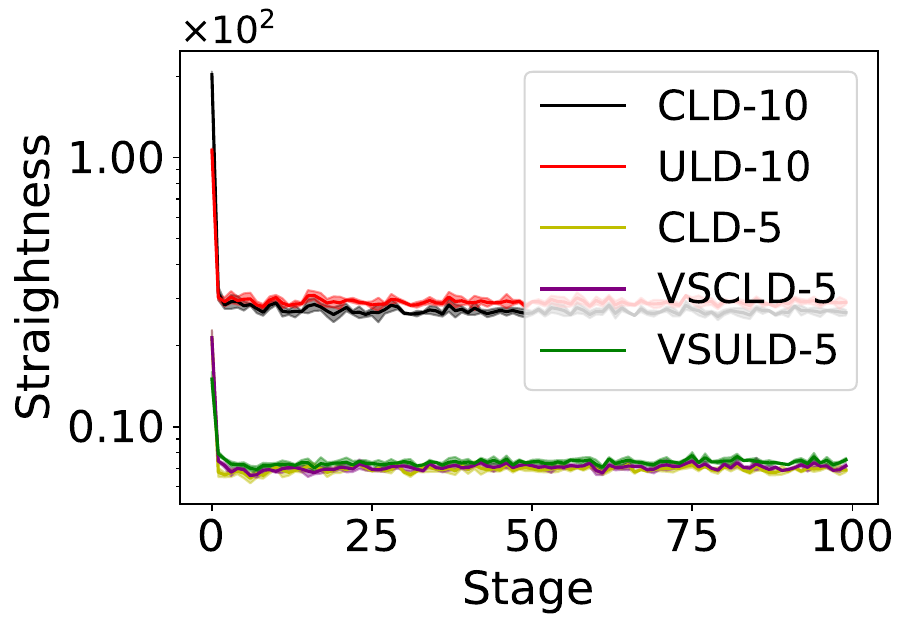}}
  \vskip -0.1in
  \caption{Straightness metric of probability flow ODEs on the non-stretched dimension via CLD, ULD, VSCLD, and VSULD.}\label{Fig:trajectories_dynamics}
  % \vspace{-0.5em}
\end{figure}

\paragraph{Overdamped v.s. Underdamped}

We also compare VSULD-5 models with VSDM using a fixed $\beta$ value ($\beta = 5$, VSDM-5) and the same VPSDE schedule as in \cite{VSDM} with $\beta_{\max}=10$ (VSDM-10 (VP)). Figure \ref{fig:underdamped_over_damped} shows that VSDM-5 and VSDM-10 (VP) are overall comparable, and VSULD consistently outperforms the overdamped alternatives in terms of accuracy and speed.

\begin{figure}[!ht]
  % \centering
  % \vspace{-0.07in}
  \subfigure[\small{Spiral-8Y}]{\includegraphics[scale=0.27]{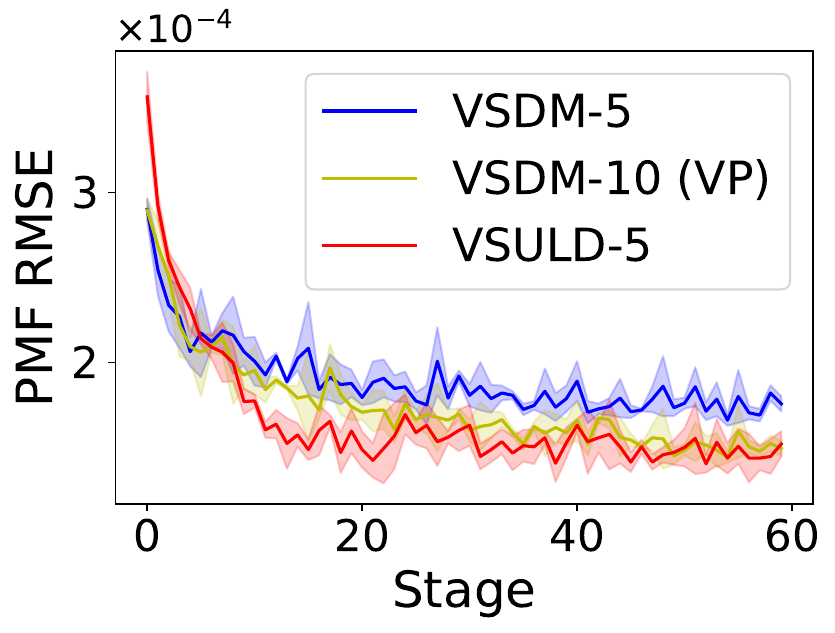}}
  \subfigure[\small{Checkerboard-6X}]{\includegraphics[scale=0.27]{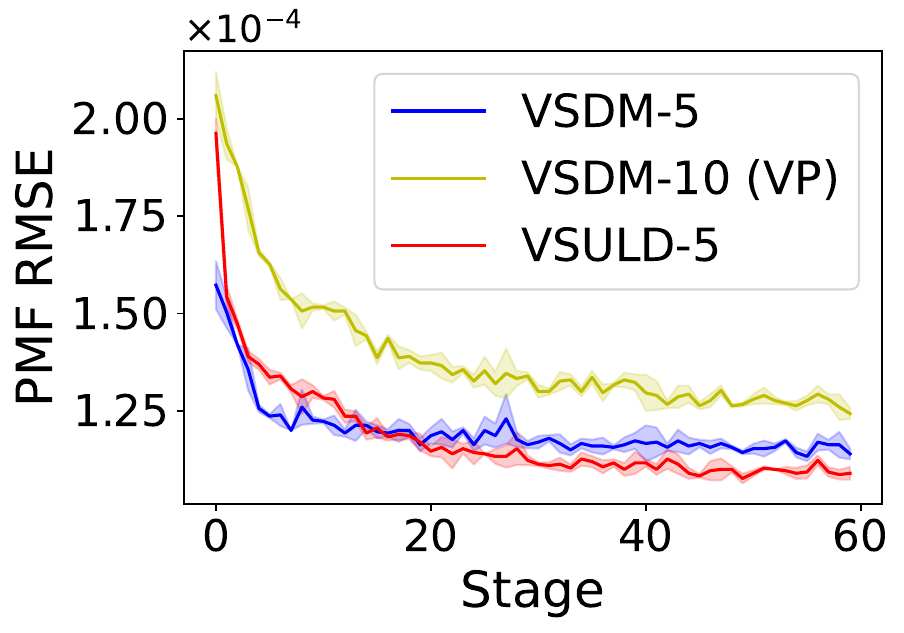}} \vskip -0.1in
  \caption{Overdamped versus underdamped models.}\label{fig:underdamped_over_damped}
  % \vspace{-0.5em}
\end{figure}

%% file: chapters/exp_2_time_series.tex
\subsection{Time Series Forecasting}
We demonstrate our models ability in a real world multivariate probabilistic forecasting. Given a sequence $x_{1:N} = \{(t_i, x_i)\}_{i=1}^N$ where $t_i\in\R$ is a time variable and $x_i \in \R^d$. Our goal is to predict the next elements of this sequence, that is predict $x_{N+1}, \dots, x_{N+P}$ for some time points $t_{N+1}, \dots t_{N+P}$. 

We follow the same approach as in \cite{VSDM} and encode the sequence $x_{1:N}$ into a vector $h_i \in \R^h$. We then train a conditional  diffusion model to predict $x_{n+1} | h_n$. Such a model allows generating the entire prediction sequence in an auto-regressive fashion as in \cite{rasul2021autoregressive}.

We utilize a similar U-Net architecture as as in \cite{VSDM}. We use a second order Heunn method as introduced in \cite{EDM}. To the best of our knowledge the use of second order samplers had not been explored in the time series forecasting problem. As expected, this change significantly improves the forecasts. Due to the autoregressive nature of the method it is important to reduce the error in early stages to prevent the model from drifting away. The training is performed on a laptop Geforce RTX 4070 with $8GB$ of VRAM.

\begin{table}[h!]
    \centering
    \begin{tabular}{lccc}
        \toprule
        \textbf{} & \textbf{Electricity} & \textbf{Exchange} & \textbf{Solar} \\
        \midrule
        \textbf{CLD}   & 0.2115          & \textbf{0.0069}  & 0.4891 \\
        % \textbf{VP}    & 0.0399          & 0.0069           & \textbf{0.3658} \\
        \textbf{VSDM}  & 0.0492          & 0.0070           & 0.4726 \\
        \textbf{VSCLD} & 0.0575          & 0.0137           & 0.5325 \\
        \textbf{VSULD} & \textbf{0.0398} & 0.0098           & \textbf{0.4628} \\
        \bottomrule
    \end{tabular}
    \caption{Performance comparison on Electricity, Exchange, and Solar in the CRPS-Sum metric}
    \label{tab:crps_sum}
\end{table}

% \Wei{we may consider not to report VP method. because VP seems to be better than VSULD and VSCLD in both versions. }

% \Wei{is CLD baseline in exchange \textbf{0.0007} or \textbf{0.007}?}

We test on the exchange rate dataset which contains $6071$ $8$-dimensional measurements every day. The solar dataset is a $137$ dimensional dataset with $7009$ values measured every hour. Finally the electricity dataset is an hourly dataset with $370$ dimensions with $5833$ measurements. In Table \ref{tab:crps_sum} we demonstrate the value of CRPS-Sum of our method. We compare against CLD, and VDSM using the same architecture and the improved sampler.  We present forecasts in the first three dimensions for the solar dataset in figure \ref{Fig:forecasts-solar-3dims}, forecasts for other datasets and methods in the appendix. 

\begin{figure}[!ht]
  % \centering
  % \vspace{-0.07in}
  \includegraphics[width=\linewidth]{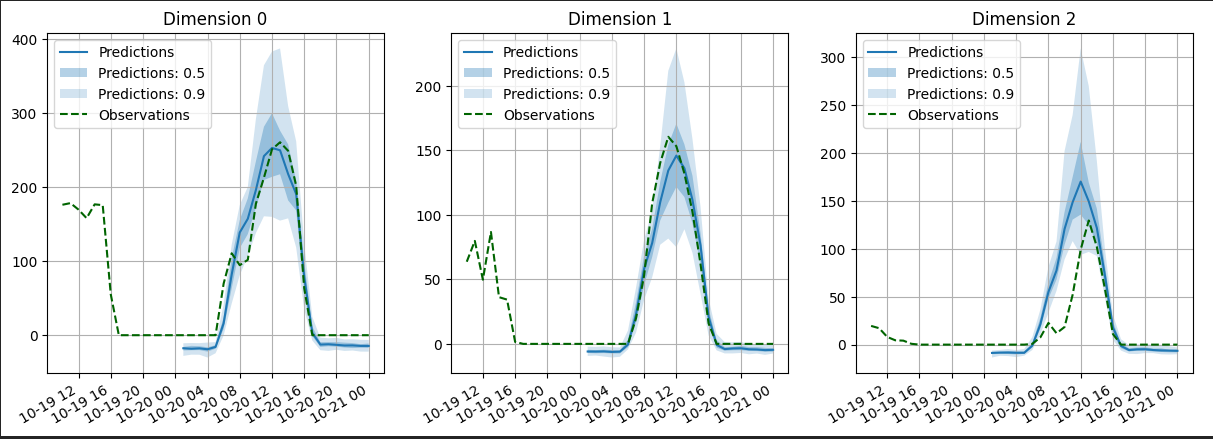}
  % \vskip -0.1in
  \caption{Sample forecasts of our method in the solar dataset} \label{Fig:forecasts-solar-3dims}
  % \vspace{-0.5em}
\end{figure}

%% file: chapters/exp_3_image_generation.tex
\subsection{Image experiments}
\textbf{Experiment Setup} We test the scalability of our method by training an unconditional generative model on the CIFAR-10 dataset. We make use of the critical damping transformation to perform this experiment. We train our model in $8$ NVIDIA V100-16GB GPUs with a batch size of $256$. We follow standard practices and use the EMA during inference where we made use of the second order Heun's method to discretize the probability flow ODE. We present some sample images in Figure \ref{Fig:cifar-samples}. 

One natural concern is that when the variational score gets updated, this changes the dynamics and the target distribution of the backwards process. Then the score needs to be correctly updated to revert for these changes. To circunvent this we make use of a step learning rate schedule with parameter $.99$. Doing so allows to keep the variational scores from changing drastically towards the later parts of training. This in combination with the stochastic approximation technique described in \ref{sec:backward_sde} allows for a stable training and annealing of the variational scores.

\begin{figure}[!ht]
  % \centering
  % \vspace{-0.07in}
  \includegraphics[width=\linewidth]{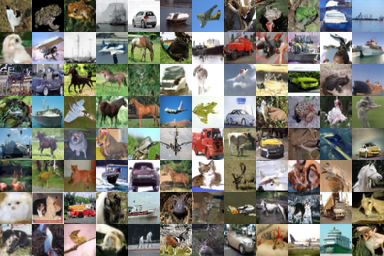}
  % \vskip -0.1in
  
  \caption{Unconditional generated samples using VSCLD on CIFAR-10} \label{Fig:cifar-samples}
  % \vspace{-0.5em}
\end{figure}

%% file: chapters/conclusions.tex
\section{Conclusions and Future Works}

Momentum Schrödinger bridge diffusion models provide a principled framework for studying generative models with optimal transport properties. However, achieving optimal transportation plans is often prohibitively expensive in real-world scenarios. To address the scalability issue, we propose the Variational Schrödinger Momentum Diffusion (VSMD) model, a scalable multivariate diffusion model that enables simulation-free training of backward scores, and the forward scores are optimized adaptively for more efficient transportation plans. Motivated by kinetic (second-order) Langevin dynamics, the inclusion of velocity components enhances training and sampling efficiency and eliminates the need for a complex denoising process. For future work, we aim to further simplify the forward diffusion process while maintaining efficient transportation plans to support more scalable applications.

\section*{Acknowledgements}

We thank the anonymous reviewers and area chairs for their insightful feedback and suggestions.

%% file: chapters/append_theory.tex
\section{Kinetic (second-order) Langevin Dynamics}

\subsection{Different Damping Regimes}
We follow \cite{CLD} and study the different damping regimes for the multivariate kinetic Langevin diffusion process:
\begin{align}\label{CLD_dyn_appendix}
   \begin{pmatrix}
    \dd \bx_t \\
    \dd \bv_t
   \end{pmatrix} &= 
   \underbrace{\frac{\beta}{2} \begin{pmatrix}
    \bv_t \\
    -(1- 2\gamma \bA_{\bx,t})\bx_t 
   \end{pmatrix}\dd t}_{\text{Hamiltonian component}} + \underbrace{\frac{\beta \gamma}{2} \begin{pmatrix} \bm{0}\\ -\big(1 - 2 \bA_{\bv,t}\big) \bv_t \end{pmatrix} \dd t + \begin{pmatrix} \bm{0} \\ \sqrt{\beta\gamma} \bI_{d}\end{pmatrix} \dd\bw_t}_{\text{Ornstein-Uhlenbeck process: O}}.
\end{align}

\begin{itemize}
    \item overdamped Langevin dynamics (LD): a high friction limit of \eqref{CLD_dyn_appendix} without momentum (Hamiltonian component) acceleration.  LD requires $\Omega(d/\epsilon^2)$ iterations to achieve an $\epsilon$ error in 2-Wasserstein (W2) distance for strongly log-concave distributions \citep{Dalalyan18}.
    \item critically-damped Langevin dynamics (CLD) via $\rmR=1$ in Eq.\eqref{damp_transform}: theoretically optimal trade-off between mass oscillation and damping \citep{classical_mechanics}. However, in practice, we may need to tune the damping ratio $\rmR$ to achieve the best balance between acceleration and transport efficiency.
    \item underdamped Langevin dynamics (ULD) via $\rmR<1$: the Hamiltonian component plays a crucial role and also induces more oscillatory behavior. ULD requires only $\Omega(\sqrt{d}/\epsilon)$ (instead of $\Omega(d/\epsilon^2)$ via LD) iterations to achieve an $\epsilon$ error in W2 for strongly log-concave distributions \citep{ULD_convergence}. 
\end{itemize}

% \textcolor{blue}{should be great if we can include a figure here}

\subsection{Numerical Schemes}

The Euler–Maruyama scheme for the backward kinetic Langevin diffusion in Eq.\eqref{backward_process} %(12)
follows that
\begin{equation}
\begin{split}
\label{EM_scheme}
\overleftarrow\bx_{(n-1)h}&=\overleftarrow\bx_{nh} - \frac{h}{2}\beta  \overleftarrow\bv_{nh} \\
\overleftarrow\bv_{(n-1)h}&=\overleftarrow\bv_{nh} + \frac{h}{2}\beta (1-2\gamma \bA_{\bx,nh}) + \frac{h}{2}\beta\gamma \big(1 - 2 \bA_{\bv,nh}\big) \overleftarrow\bv_{nh} +h\beta\gamma s_{nh}(\overleftarrow\ba_{nh})+\sqrt{\beta\gamma h} \bm{\overleftarrow\xi}_{nh},
\end{split}
\end{equation}
where $h$ is the learning rate.

Theoretically, the Euler–Maruyama scheme \eqref{EM_scheme} suffers from instability with large discretization step sizes. Motivated by the symplectic Euler for the Hamiltonian system, we consider the symmetric splitting (S2) scheme \citep{tuckerman2010statistical, leimkuhler2013rational, CLD} for the kinetic Langevin dynamics to ensure better stability. To that end, we first compose the Hamiltonian component in Eq.\eqref{CLD_dyn_appendix} into two parts:
\begin{align}
\label{CLD_dyn_discrete_scheme}
   \begin{pmatrix}
    \dd \overleftarrow\bx_t \\
    \dd \overleftarrow\bv_t
   \end{pmatrix} &= 
   \underbrace{\frac{\beta}{2} \begin{pmatrix}
    \overleftarrow\bv_t \\
    \bm{0} 
   \end{pmatrix}\dd t}_{\text{A}} + \underbrace{\frac{\beta}{2} \begin{pmatrix}
    \bm{0} \\
    -(1- 2\gamma \bA_{\bx,t})\overleftarrow\bx_t 
   \end{pmatrix}\dd t}_{\text{B}} + \underbrace{\frac{\beta\gamma}{2} \begin{pmatrix} \bm{0}\\ -\big(1 - 2 \bA_{\bv,t}\big) \overleftarrow\bv_t -2   s_t(\overleftarrow\ba_t)\end{pmatrix} \dd t + \begin{pmatrix} \bm{0} \\ \sqrt{\beta\gamma} \bI_{d}\end{pmatrix} \dd\overleftarrow\bw_t}_{\text{O}},
\end{align}
where each part yields an ``analytic'' form and the underlying Kolmogorov (Fokker-Planck) operators are denoted by $\mathcal{L}_{A}$, $\mathcal{L}_{B}$, and $\mathcal{L}_{O}$, respectively.

The stochastic discretization schemes for kinetic Langevin are well-studied and lead to different formulations. The main difference lies in the approximations of the Hamiltonian component \citep{leimkuhler2013rational}
\begin{itemize}
    \item the BAOAB method: $\Phi_{\text{BAOAB}}^{h}= \exp(\frac{h}{2}\mathcal{L}_{B}) \exp(\frac{h}{2}\mathcal{L}_{A}) \exp(h\mathcal{L}_{O}) \exp(\frac{h}{2}\mathcal{L}_{A})\exp(\frac{h}{2}\mathcal{L}_{B})$;
    \item the ABOBA method: $\Phi_{\text{ABOBA}}^{h}=\exp(\frac{h}{2}\mathcal{L}_{A}) \exp(\frac{h}{2}\mathcal{L}_{B})\exp(h\mathcal{L}_{O})\exp(\frac{h}{2}\mathcal{L}_{B})\exp(\frac{h}{2}\mathcal{L}_{A})$.
\end{itemize}

In particular, the ABOBA method based on the symmetric splitting scheme follows that:
\begin{align}\label{S2_aboba}
   \overleftarrow\bx_{(k-\frac{1}{2})h}&=\overleftarrow\bx_{nh} - \frac{h}{4}\beta \overleftarrow\bv_{nh} \notag \\
   \overleftarrow\bv_{(k-\frac{1}{2})h}&=\overleftarrow\bv_{nh} + \frac{h}{4}\beta (1-2\gamma \bA_{\bx,nh}) \notag \\
   \overleftarrow\bv_{(k-\frac{1}{2})h}&=\overleftarrow\bv_{(k-\frac{1}{2})h}+\frac{h}{2}\beta\gamma \big(1 - 2 \bA_{\bv,nh}\big) \overleftarrow\bv_{nh} +h\beta\gamma s_{nh}(\overleftarrow\ba_{nh})+\sqrt{\beta\gamma h} \bm{\overleftarrow\xi}_{nh}\notag \\
   \overleftarrow\bv_{(n-1)h}&=\overleftarrow\bv_{(k-\frac{1}{2})h} +\frac{h}{4}\beta (1-2\gamma \bA_{\bx,(k-\frac{1}{2})h})\notag \\
   \overleftarrow\bx_{(n-1)h}&=\overleftarrow\bx_{(k-\frac{1}{2})h} - \frac{h}{4}\beta \overleftarrow\bv_{(n-1)h}.\notag
\end{align}

The BAOAB method can derived similarly. The numerical study of the invariant measure is based on the Baker–Campbell–Hausdorff (BCH) expansion \citep{hairer2006geometric}. The symmetric splitting scheme is a second-order integrator and is known to yield an approximation error of $O(h^3)$ \citep{leimkuhler2013rational, Chen15, CLD}. In contrast, the Euler–Maruyama scheme has a weaker error of $O(h^2)$ \citep{Chen15, Sitan_22_sampling_is_easy}. Nonetheless, the empirical advantage of the symmetric splitting scheme mainly holds with a small learning rate \citep{leimkuhler2013rational} and may not necessarily reduce the number of function evaluations in practice. For the sake of convenience in theoretical analysis, we will focus on the Euler–Maruyama scheme similar to the analysis in \cite{Sitan_22_sampling_is_easy}.

\section{Convergence Theory}

We follow the methodology outlined in \cite{VSDM} and utilize stochastic approximation (SA) techniques \citep{RobbinsM1951} to evaluate the generation quality based on the adaptive momentum diffusion models. By employing simulated backward trajectories, we optimize the variational scores. Consequently, the optimized forward process becomes not only simulation-free but also more transport-efficient. The iterates are conducted alternatingly and eventually yield more accurate backward score functions.

\begin{algorithm*}
   \caption{The SA formulation of the variational Schr\"odinger diffusion models. We approximate $\nabla_{\bv} \log  \overrightarrow{\rho}_{t}^{(k)}$ through the parametrized score estimation $s^{(k+1)}_{t}$ at each stage $k$ and time $t$. }
   \label{alg:SA_algorithm}
\begin{algorithmic}
\REPEAT
   \STATE{\textbf{Simulation}: Draw approximate samples $(\overleftarrow\bx_{(n-1)h}^{(k+1)},
    \overleftarrow\bv_{(n-1)h}^{(k+1)})$ from the backward process \eqref{EM_scheme} with fixed $\bA_{\bx, nh}^{(k)}$ and $\bA_{\bv, nh}^{(k)})$, where $(\overleftarrow\bx_{(N-1)h}^{(k+1)},
    \overleftarrow\bv_{(N-1)h}^{(k+1)})\sim \mathrm{N}(\bm{0}, \bSigma^{(k)}_{(N-1)h|0})$ and $\bSigma^{(k)}_{(N-1)h|0}$ is defined in Eq.(6), $n\in \{1,2,\cdots, N-1\}$.}
   \STATE{\textbf{Optimization}: } Minimize the transport cost via the forward loss function \eqref{MDM_loss}:
   \begin{equation}\label{SA_iterates}
       \bA_{\ba, nh}^{(k+1)}=\bA_{\ba, nh}^{(k)}-\eta_{k+1} \nabla \overrightarrow{\mathcal{L}}_{nh}(\bA_{\ba, nh}^{(k)}; \overleftarrow\bx^{(k+1)}_{nh}),
   \end{equation}
   where $\eta_k$ denotes the step size and $n\in\{0, 1, \cdots, N-1\}$.
   \UNTIL{$k=k_{\max}$}
    \vskip -1 in
\end{algorithmic}
\end{algorithm*}

In our convergence study, we assume a single step of sampling and a single step of optimization and conduct the iterates in Eq.\eqref{SA_iterates} for the coupled score function $\bA_{\ba, t}$ instead of $\bA_{\bx, t}$ for theoretical convenience. However, this simplification is not required in practical applications to boost the performance.

The SA iterates \eqref{SA_iterates} can be viewed as a stochastic numerical scheme of an ODE system as follows
\begin{equation}
\label{mean_field_perturbed_ode}
    \dd \bA_{\ba, t} = \nabla \overrightarrow{\bL}_t(\bA_{\ba, t})\dd s,
\end{equation}
where $\nabla \overrightarrow{\bL}_t(\bA_{\ba, t})$ is the mean-field aggregated from random-field functions $\nabla \overrightarrow{\mathcal{L}}_t(\bA_{\ba, t}; \overleftarrow\ba^{(\cdot)}_t)$:
\begin{equation}
\begin{split}
\label{mean_field_perturbed}
\nabla  \overrightarrow{\bL}_t(\bA_{\ba, t})&=\int_{\MX} \nabla \overrightarrow{\mathcal{L}}_t(\bA_{\ba, t}; \overleftarrow\ba^{(\cdot)}_t)  \overleftarrow{\rho}_{t}(\dd\overleftarrow\ba^{(\cdot)}_t).
\end{split}
\end{equation}

%%%%%%%%%%%%%%%%%%%%%%%%%%%%%% before edits %%%%%%%%%%%%%%%%%%%%%%%%%%%%%%

We aim to find the solution of $\nabla \overrightarrow{\bL}_t(\bA_{\ba, t}^{\star})=\bm{0}$ through approximate samples $(\overleftarrow\bx_{(n-1)h}^{(k+1)}, \overleftarrow\bv_{(n-1)h}^{(k+1)})$ from the backward process \eqref{EM_scheme}. 

Next, we present the underlying assumptions for regularity conditions for the solution $\bA_{\ba, t}^{\star}$ and the neighborhood. %We consider time invariant cases with $\bA_{\ba, t}=(\bA_{\bx, t}, \bA_{\bv, t}):=(\bA_{\bx}, \bA_{\bv})$.

\begin{assump}[Positive Definiteness]
\label{ass_pd}
Both $\bI_d - 2 \gamma \bA_{\bx, t}$ and $\bI_d - 2 \bA_{\bv, t} $ are symmetric and positive-definite. Morever, $\|\bD\|_{op}\lesssim O(1)$, where $\|\cdot\|_{op}$ denotes the operator norm of a matrix.
\end{assump}

\begin{assump}[Locally strong convexity]
\label{ass_local_state_space}
For any stable equilibrium $\bA_{\ba}^{\star}$ with $\nabla \overrightarrow{\bL}_t(\bA_{\ba}^{\star})= \bm{0}$, there exists a convex set $\bTheta$ s.t. $\bA_{\ba}^{\star}\in \bTheta\subset \mathcal{A}$ and $m\bI \preccurlyeq \frac{\partial^2 \overrightarrow{\bL}_t}{\partial \bA^2}(\bA)\preccurlyeq M \bI$ for $\forall \bA \in \bTheta$ and some fixed constants $M>m>0$.
\end{assump}

The following assumes the smoothness of the score functions with respect to the input $\bx, \by$ and variational scores $\bA_{\ba_1}, \bA_{\ba_2}$  and similar ones have been widely used in \cite{lee2022convergence, Sitan_22_sampling_is_easy, chen2023improved, VSDM}.

\begin{assump}[Smoothness]
\label{ass_smoothness}
There exists a fixed constant $L$ such that for any $t\in[0, T]$, $\bA_{\ba_1},\bA_{\ba_2}\in\mathcal{A}$ and $\bx, \by \in \mathcal{X}$, the score functions $\nabla_{\bv}\log  \overrightarrow{\rho}_{\ba_1, t}$ and $\nabla_{\bv}\log  \overrightarrow{\rho}_{\ba_2, t}$ w.r.t. $\bA_{\ba_1}$ and $\bA_{\ba_2}$ satisfy 
\begin{equation*}
\begin{split}
        \|\nabla_{\bv}\log  \overrightarrow{\rho}_{\ba_1, t}(\bx)-\nabla_{\bv}\log  \overrightarrow{\rho}_{\ba_2, t}(\by)\|_2 &\leq L \|\bx-\by\|_2 + L \|\bA_{\ba_1}-\bA_{\ba_2}\|, \\
\end{split}
\end{equation*}
where $\|\cdot\|_2$ is the Euclidean norm and $\|\cdot\|$ denotes the standard matrix norm.
\end{assump}

\begin{assump}[Bounded Second Moment]
\label{ass_moment_bound}
The second moment of $\rho_{\text{data}}$ is upper bounded by $\mathfrak{m}_2^2$.
\end{assump}

\begin{assump}[Estimation of Score Functions]
\label{ass_score_estimation}

For all $t\in[0, T]$, and any $\bA_{\ba}$, the estimation error of the score functions is upper bounded by $\epsilon_{\text{score}}^2$:
$$\E_{\overrightarrow{\rho}_{t}}[\|s_{t}- \nabla_{\bv} \log \overrightarrow{\rho}_{\ba, t}\|_2^2]\leq \epsilon_{\text{score}}^2.$$

\end{assump}

\paragraph{Proof Sketch} Similar to \cite{VSDM}, the understanding of the quality of the adaptively generated data hinges on the stochastic approximation framework and can be decomposed into three steps: 
\begin{itemize}
    \item \textbf{Fixed Generation Quality}: We first show that given a fixed $\bA_{\ba, t}^{(k)}$, the generated data is approximately close to the real data in Theorem \ref{theorem:quality_of_data};
    \item \textbf{Convergence of Variational Scores}: We next prove the convergence of $\bA_{\ba, t}^{(k)}$ to the optimal $\bA_{ \ba, t}^{\star}$ via stochastic approximation in Theorem \ref{theorem_L2_convergence};
    \item \textbf{Adaptive Generation Quality:} In the limit of infinite iterations, we show the generated data is close to the real data in distribution given the optimal $\bA_{ \ba, t}^{\star}$ in Theorem \ref{theorem_adaptive_sampling}.
\end{itemize}
In particular, the stochastic approximation part is standard and inherited from \cite{VSDM}. The major novelty lies in the extension of \textcolor{dark2blue}{single-variate kinetic} Langevin diffusion to \textcolor{dark2blue}{multi-variate kinetic} Langevin diffusion through a customized Lyapunov function in Eq.\eqref{Lyapunov_Function}. As such, the adaptive momentum diffusion differs from the vanilla CLD in that the transportation plans are optimized locally in particular tailored to the data, moreover, the training maintains the same efficiency as CLD due to the simulation-free property of forward processes.

\subsection{Fixed Generation Quality}
In this section, we first study the generation quality based on time-invariant $\bA_{\bx, t}:=\bA_{\bx}$ and $\bA_{\bv,t}:=\bA_{\bv}$ and discuss the extensions to time-variant cases. 
% For simplicity, we assume that $\bD_t$ in~\eqref{FB-SDE-linear-unified} is time-invariant, i.e.
\begin{align}
\mathrm{d} \overrightarrow\ba_t&=-\frac{1}{2}\bD\beta \overrightarrow\ba_t \dd t + \bbg\dd \overrightarrow\bw_t \label{eq:ULD-time-invariant}\\
        \bD & = \begin{pmatrix}
        0 & -1 \\
        1 & \gamma 
        \end{pmatrix}\otimes \bI_d-2\gamma \begin{pmatrix}
            0 & 0\\ \bA_{\bx} & \bA_{\bv}
        \end{pmatrix}. \notag
\end{align}
We denote the distribution of $\overrightarrow \ba_t$ by $\overrightarrow\rho_t$ and its $\bx$ and $\bv$-marginal by $\overrightarrow\rho_{\bx, t}$ and $\overrightarrow\rho_{\bv, t}$ respectively. To ensure an exponential convergence of CLD in~\eqref{eq:ULD-time-invariant} to its invariant measure, we make the following assumption:

\begin{lemma}[Invariant Measure]\label{lem:invariant_measure}
If Assumption~\ref{ass_pd} holds, the invariant measure of~\eqref{eq:ULD-time-invariant} is given by
\begin{align*}
    \mu = \mathrm{N}\bigg(\bm{0}, \begin{pmatrix}
        \bB_1^{-1} & \bm{0} \\ \bm{0} & \bB_2^{-1}
    \end{pmatrix}\bigg).
\end{align*}
where $\bB_1 = (\bI_d - 2 \gamma \bA_{\bx})^T(\bI_d - 2 \bA_{\bv})$ and $\bB_2 = \bI_d - 2 \bA_{\bv} $
\end{lemma}
\begin{proof}
Recall the Fokker Planck Equation gives us
\begin{align*}
    \partial_t\overrightarrow \rho_t(\ba) = \nabla\cdot\bigg(\overrightarrow \rho_t(\ba)\big(\frac12 \bD \beta \ba + \frac12 \bbg\bbg^T\nabla\ln \overrightarrow \rho_t(\ba)\big)\bigg),
\end{align*}
where $\overrightarrow \rho_t$ is the density of $\overrightarrow \ba_t$. Note that $\mu(\ba)=\mu(\bx, \bv) \propto \exp(-\frac{\bx^T \bB_1 \bx}{2} - \frac{\bv^T \bB_2 \bv}{2})$. Hence if $\overrightarrow\rho_t = \mu$, then
\begin{align*}
    \partial\overrightarrow \rho_t(\ba) & = \nabla\cdot\bigg(\mu(\ba)\bigg(\frac12 \bD \beta \ba + \frac12 \bbg\bbg^T\nabla\ln \mu(\ba)\bigg)\bigg) \\
    & = \bigg\langle \nabla\mu(\ba), \frac12 \bD \beta \ba + \frac12 \bbg\bbg^T\nabla\ln \mu(\ba)\bigg\rangle + \mu(\ba) \nabla\cdot\bigg(\frac12 \bD \beta \ba + \frac12 \bbg\bbg^T\nabla\ln \mu(\ba)\bigg)\\
    & = \bigg\langle\mu(\ba) \begin{pmatrix}
        -\bB_1\bx \\ -\bB_2 \bv
    \end{pmatrix}, \frac{\beta}{2} \begin{pmatrix}
        0 & -\bI_d \\ \bI_d - 2\gamma\bA_\bx & \gamma \bB_2
    \end{pmatrix}\begin{pmatrix}
        \bx \\ \bv
    \end{pmatrix} + \frac{\beta}{2}\begin{pmatrix}
        0 \\ \gamma \bB_2 \bv
    \end{pmatrix}\bigg\rangle + \frac{\beta}{2}\mu(\ba)(\mathrm{Tr}(\bD) - \mathrm{Tr}(\bB_2))\\
    & = \frac{\beta}{2}\mu(\ba)\bigg\langle \begin{pmatrix}
        -\bB_1\bx \\ -\bB_2 \bv
    \end{pmatrix}, \begin{pmatrix}
        -\bv \\ (\bI_d - 2\gamma\bA_\bx)\bx 
    \end{pmatrix}\bigg\rangle \\
    & = 0.
\end{align*}
Therefore $\mu$ is an invariant measure of~\eqref{eq:ULD-time-invariant}. This invariant measure is unique since~\eqref{eq:ULD-time-invariant} is a linear SDE.
\end{proof}
We denote the distribution of the numerical reverse process~\eqref{EM_scheme} by $\overleftarrow\rho_t$ and its $\bx$ and $\bv$-marginal distribution by $\overleftarrow\rho_{\bx, t}$ and $\overleftarrow\rho_{\bv, t}$ respectively. We aim to bound $\mathrm{TV}(\overleftarrow\rho_{\bx, 0}, p_{\text{data}})$. To achieve this, we first bound $\mathrm{TV}(\overleftarrow\rho_0, \overrightarrow{\rho}_0)$ and then apply the Data-Processing Inequality. We lay out three standard assumptions following \citet{Sitan_22_sampling_is_easy} to conduct our analysis.

% \begin{assump}[Lipschitz Score]
% \label{ass_smoothness}
% The score function $\nabla_v \log  \overrightarrow{\rho}_{t}$ is $L$-Lipschitz for all $t\geq 0$.
% \end{assump}

% \begin{assump}[Second Moment Bound]
% \label{ass_moment_bound}
% The data distribution has a bounded second moment $\mathfrak{m}_2^2:=\E_{\rho_{\text{data}}}[\|\cdot\|_2^2]<\infty$.
% \end{assump}

% \begin{assump}[Score Estimation Error]
% \label{ass_score_estimation}

% For all $t\in[0, T]$, and any $\bA_t$, we have some estimation error . 
% $$\E_{\overrightarrow{\rho}_{t}}[\|s_{t}- \nabla_v \log \overrightarrow{\rho}_{t}\|_2^2]\leq \epsilon_{\text{score}}^2.$$

% \end{assump}

\begin{theorem}[Fixed Generation Quality]\label{theorem:quality_of_data}
    Assume assumptions \ref{ass_pd}, \ref{ass_smoothness}, \ref{ass_moment_bound} and \ref{ass_score_estimation} hold. The generated data distribution is close to the data distributions $p_{\text{data}}$ such that
    \begin{equation*}
        \mathrm{TV}(\overleftarrow{\rho}_{\bx,0}, p_{\text{data}})\lesssim \underbrace{\sqrt{\mathrm{KL}(p_{\text{data}}\|\mu_\bx) + \mathrm{FI}(p_{\text{data}}\|\mu_\bx)} \exp(-T)}_{\text{convergence of forward process}} + \underbrace{(L\sqrt{dh} +  L\mathfrak{m}_2 h)\sqrt{T}}_{\text{ discretization error}} + \underbrace{\epsilon_{\text{score}}\sqrt{T}}_{\text{score estimation}},
    \end{equation*}

    where $\mu_\bx$ is the $\bx$-marginal distribution of $\mu$ defined in Lemma~\ref{lem:invariant_measure}.
\end{theorem}
\begin{proof}
Following~\citet{chen2023improved}, we employ the chain rule for KL divergence and obtain:
\begin{align*}
\mathrm{KL}(\overrightarrow\rho_0\|\overleftarrow{\rho}_{0}) \leq \mathrm{KL}(\overrightarrow{\rho}_{T} \| \overleftarrow{\rho}_{T}) + \E_{\overrightarrow{\rho}_{T}(\ba)}[\mathrm{KL}(\overrightarrow\rho_{0|T}(\cdot\|\ba)| \overleftarrow{\rho}_{0|T}(\cdot\|\ba)],
\end{align*}
where $\overrightarrow\rho_{0|T}$ is the conditional distribution of $\ba_0$ given $\ba_T$ and likewise for $\overleftarrow{\rho}_{0|T}$. Note that the two terms correspond to the convergence of the forward and reverse process respectively. We proceed to prove that 
\begin{align*}
    & \text{Part I: Forward process}\quad\quad \mathrm{KL}(\overrightarrow{\rho}_{T} \| \overleftarrow{\rho}_{T})%=\mathrm{KL}(\overrightarrow{\rho}_{T} \| \overrightarrow{\rho}_{T}^{\circ}) 
    \lesssim (\mathrm{KL}(p_{\text{data}}\|\mu_\bx) + \mathrm{FI}(p_{\text{data}}\|\mu_\bx)) e^{-T},\\
     & \text{Part II: Backward process}\quad \E_{\overrightarrow{\rho}_{T}(\bx)}[\mathrm{KL}(\overrightarrow\rho_{0|T}(\cdot|\bx)\| \overleftarrow{\rho}_{0|T}(\cdot|\bx)] \lesssim (L^2dh + L^2 m_2^2 h^2)T + \epsilon_{\text{score}}^2 T.
\end{align*}

\text{Part I:} Note that $\overleftarrow \rho_T = \mu$, where $\mu$ is the invariant measure of~\eqref{eq:ULD-time-invariant}. Following~\cite{ma2021there}, we construct the Lyapunov Function
\begin{align}\label{Lyapunov_Function}
    \mathrm{L}(\overrightarrow\rho_t) := \mathrm{KL}(\overrightarrow\rho_t|\mu) + \E_{\overrightarrow\rho_t}\bigg[\bigg\langle \nabla\ln\frac{\overrightarrow\rho_t}{\mu}, S \nabla\ln\frac{\overrightarrow\rho_t}{\mu} \bigg\rangle\bigg]
\end{align}
for some positive definite matrix $S$. Since the Gaussian distribution $\mu$ satisfies the log-Sobolev inequality, one can show that there exists a constant $c > 0$ such that $\frac{d}{dt} \mathrm{L}(\overrightarrow{\rho}_t) \leq -c \mathrm{L}(\overrightarrow{\rho}_t)$. Here $c$ depends on $\beta$, $\gamma$ and the log-sobolev constant of $\mu$. Thus,
\begin{align*}
    \text{KL}(\overrightarrow\rho_T|\mu)\leq \mathrm{L}(\overrightarrow\rho_T) \leq \mathrm{L}(\overrightarrow\rho_0)e^{-cT} \lesssim (\mathrm{KL}(p_{\text{data}}\|\mu_\bx) + \mathrm{FI}(p_{\text{data}}\|\mu_\bx)) e^{-T}.
\end{align*}
For the detailed proof, we refer readers to Appendix C of~\cite{ma2021there}, as the argument closely follows similar reasoning. For brevity, we omit it here.

\text{Part II:} The proof for the convergence of the reverse process is essentially identical to Theorem 15 of~\citet{Sitan_22_sampling_is_easy}, with the only potential replacements being instances of $\bigg\|\begin{pmatrix}
    0 & \bI_d\\ \bI_d & \gamma \bI_d
\end{pmatrix}\bigg\|_{op}$ with $\|\bD\|_{op}$. However, they are equivalent due to Assumption~\ref{ass_pd}. Therefore, we omit the proof here.

Combining the results from Parts I and II, and applying Pinsker's Inequality, we obtain
\begin{align*}
    \mathrm{TV}(\overleftarrow{\rho}_{0}, \overrightarrow{\rho}_{0})\lesssim \sqrt{\mathrm{KL}(p_{\text{data}}\|\mu_\bx) + \mathrm{FI}(p_{\text{data}}\|\mu_\bx)} \exp(-T) + (L\sqrt{dh} +  L\mathfrak{m}_2 h)\sqrt{T} + \epsilon_{\text{score}}\sqrt{T}.
\end{align*}
The final result follows by applying the Data-Processing Inequality to transition from $\overrightarrow{\rho}_0$ and $\overleftarrow{\rho}_0$ to their respective $\bx$-marginals.
\end{proof}
\begin{remark}
The proof could be potentially generalized to the case where $\bA_{\bx, t}$ and $\bA_{\bv, t}$ are time-varying. First, for the backward process, note that the proof of Theorem 15 in~\citet{Sitan_22_sampling_is_easy} relies only on the score estimation and the Lipschitz property of the score function, which does not require the drift of the forward process to be time-invariant. Second, for the forward process, for a fixed $T > 0$, one can always consider a modified version of~\eqref{FB-SDE-linear-unified} with time-averaged drift given by~\eqref{eq:ULD-time-invariant}, where $\bD = \frac{1}{T} \int_0^T \bD_t , \mathrm{d}t$. Thus, for the same initial condition $\ba_0$,~\eqref{FB-SDE-linear-unified} and~\eqref{eq:ULD-time-invariant} will generate the same distribution at time $T$.
\end{remark}

\subsection{Convergence of Variational Scores}
$\bA_{\ba, t}^{(k)}$ tracks a mean-field ODE, which converges to the equilibrium $\bA_{\ba, t}^{\star}$ if we can establish the stability condition of the mean-field ODE. As such, we can ensure that the perturbations caused by errors in Theorem \ref{theorem:quality_of_data} result in, at most, similar variations in subsequent iterates.

The following is a restatement of Lemma 2 in \cite{VSDM}

\begin{lemma}[Local stabiltity]\label{lemma_local_stability}
    Given assumptions \ref{ass_local_state_space} and \ref{ass_smoothness}, we can identify a local stability condition for any $\bA\in \bTheta$ such that
        \begin{equation*}
        \label{local_stability}
            \langle \bA-\bA_{\ba, t}^{\star}, \nabla \overrightarrow{\bL}_t(\bA) \rangle \geq m \|\bA-\bA_{\ba, t}^{\star}\|_2^2. 
        \end{equation*}
\end{lemma}

Next, we assume the step size $\eta_k$ follows the tradition in stochastic approximation \citep{Albert90}.
\begin{assump}[Step size]
\label{ass_step_size}
\begin{equation*} \label{a1}
0<\eta_{k+1}<\eta_k, \ \ \sum_{k=1}^{\infty} \eta_{k}=+\infty,\ \  \sum_{k=1}^{\infty} \eta_{k}^{2\alpha}<\infty, \ \ \alpha \in \big(\frac{1}{2}, 1\big].
% \yt{\text{Maybe add a bracket?} \lim_{k\rightarrow \infty} \inf \left(2m  \dfrac{\eta_{k}}{\eta_{k+1}}+\dfrac{\eta_{k+1}-\eta_{k}}{\eta^2_{k+1}}\right)}. % Wei: fixed.
\end{equation*}
 \end{assump}

The next result is a restatement of Theorem 2 in \cite{VSDM} to prove the convergence of the variational scores.
\begin{theorem}[Convergence in $L^2$]\label{theorem_L2_convergence}
    Given assumptions \ref{ass_local_state_space} - \ref{ass_step_size} and a large enough $k$, the variational score $\bA_{\ba, t}^{(k)}$ in algorithm \ref{alg:SA_algorithm} converges to a local equilibrium $\bA_{\ba, t}^{\star}$ that motivates efficient transport such that
    \begin{equation*}
    \E_{\overleftarrow{\rho}_{\ba, t}^{(k)}}[\|\bA_{\ba, t}^{(k)}-\bA_{\ba, t}^{\star}\|_2^2]\leq 2 \eta_{k}.
\end{equation*}
\end{theorem}

\subsection{Adaptive Generation Quality}

Theorem \ref{theorem_L2_convergence} shows that the non-optimized $\bA_{\ba, t}^{(k)}$ converges to the equilibrium $\bA_{\ba, t}^{\star}$, where the latter yields efficient transportation plans. Combining the study of the sample quality in Theorem \ref{theorem:quality_of_data} based on a fixed $\bA_{\ba, t}^{(k)}$, we can evaluate the adaptive sample quality based on $\bA_{\ba, t}^{\star}$, which yields more and more efficient transportation plans in the long-time limit. The following is a natural extension of Theorem 3 in \cite{VSDM} since both algorithms follow from the framework of multivariate diffusion:
\begin{theorem}
\label{theorem_adaptive_sampling}
Assume assumptions \ref{ass_pd}-\ref{ass_step_size} hold. The  adaptively generated sample at stage $k$ based on the equilibrium $\bA_t^{\star}$ with efficient transportation plans is close in total variation (TV) distance to the real sample such that
\begin{equation*}
    \mathrm{TV}(\overleftarrow{\rho}^{\star}_{0,\bx}, \rho_{\text{data}})\lesssim \underbrace{\sqrt{\mathrm{KL}(\rho_{\text{data}}\|\mu^\bx) + \mathrm{FI}(\rho_{\text{data}}\|\mu^\bx)} \exp(-T)}_{\text{convergence of forward process}} + \underbrace{(L\sqrt{dh} +  L\mathfrak{m}_2 h)\sqrt{T}}_{\text{ discretization error}} + \underbrace{(\epsilon_{\text{score}} +\sqrt{\eta_k})\sqrt{T}}_{\text{adaptive score estimation}}.
\end{equation*}
\end{theorem}

%% file: chapters/append_images.tex
\section{Experimental Details}
We present more details on the images experiments. We consider the same U-net architecture as that used in \cite{SGMS_beat_GAN} and implemented by \cite{EDM}. We normalize the images to the interval [-1,1] and use a horizontal flip as only data augmentation. We present a table of hyperparameters used during training (\ref{hyperparmeters}):
\begin{table}[h]
\centering
\caption{Table of hyperparameters used during training}
\label{hyperparmeters}
\begin{tabular}{l c}
\hline
Parameter        & Value  \\ \hline
Forward Score learning rate    & 3e-4   \\
Backward Score learning rate    & 3e-6   \\
EMA Beta    & $.9999$   \\
Sampling Time Steps    & $125$    \\ 
Batch Size       & $256$    \\
Damping parameter & $.9$
\end{tabular}
\end{table}

Our method results in the following table of FID values (\ref{fid}):

\begin{table}[h]
\centering
\caption{CIFAR10 evaluation using sample quality (FID)}
\label{fid}
\begin{tabular}{l l c}
\hline
\textbf{Class} & \textbf{Method} & \textbf{FID} $\downarrow$ \\ \hline
\multirow{6}{*}{OT} & VSCLD (Ours) & 2.89 \\ 
                    & VSDM (\cite{VSDM}) & 2.28 \\
                    & SB-FBSDE (\cite{forward_backward_SDE}) & 3.01 \\
                    & DOT (\cite{Tanaka2019DiscriminatorOT}) & 15.78 \\
                    & DGflow (\cite{Ansari2020RefiningDG}) & 9.63 \\ \hline
\multirow{4}{*}{SGMs} & SDE (\cite{score_sde}) & 2.92 \\
                      & CLD (\cite{CLD}) & \textbf{2.23} \\
                      & VDM (\cite{Kingma2021VariationalDM}) & 4.00 \\
                      & LSGM (\cite{vahdat2021score}) & 2.10 \\
                      & EDM (\cite{EDM}) & \textbf{1.97} \\ \hline
\end{tabular}
\end{table}
Despite that we don't reach the best FID values among the compared methods, this could be due to the lack of advanced preconditioning and data augmentation techniques like those presented in \cite{EDM}. A more detailed investigation on the best practices for training variational diffusion models would allow improvement on this end, we delay this detailed investigation for future work. However we must emphasize that this experiment demonstrates the scalability of the method in high dimensions.

%% file: chapters/append_time_series.tex
\section{Time Series Forecasts}
In this section, we present more forecasts generated using different methods

\subsection{Samples for VSCLD}

\begin{figure}[H]
  % \centering
  % \vspace{-0.07in}
  \includegraphics[width=\linewidth]{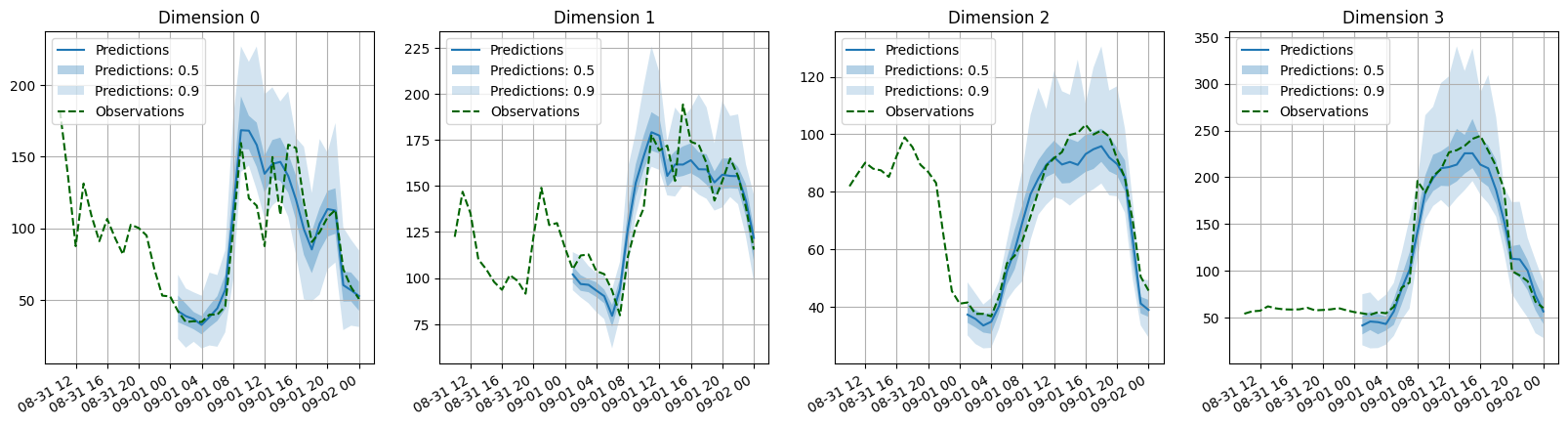}
  % \vskip -0.1in
  \caption{Sample forecasts of VSCLD in the electricity dataset} 
\end{figure}

\begin{figure}[H]
  % \centering
  % \vspace{-0.07in}
  \includegraphics[width=\linewidth]{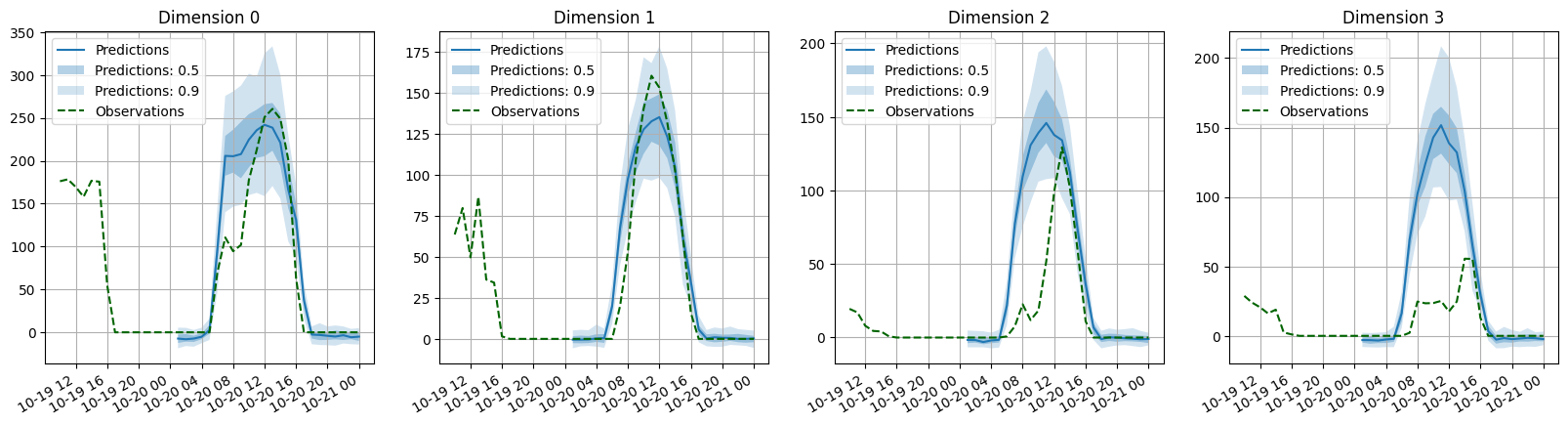}
  % \vskip -0.1in
  \caption{Sample forecasts of VSCLD in the solar dataset} 
\end{figure}

% \Wei{the predictions in the middle below do not look good}
\begin{figure}[H]
  % \centering
  % \vspace{-0.07in}
  \includegraphics[width=\linewidth]{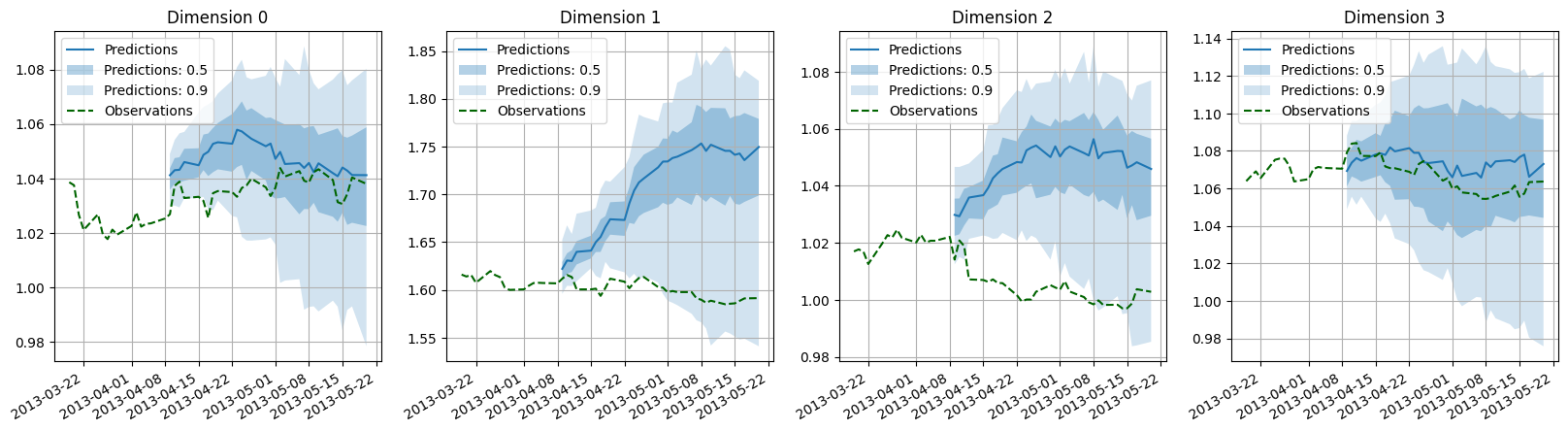}
  % \vskip -0.1in
  \caption{Sample forecasts of VSCLD in the exchange rate dataset} 
\end{figure}

\subsection{Samples for VSULD}

\begin{figure}[H]
  % \centering
  % \vspace{-0.07in}
  \includegraphics[width=\linewidth]{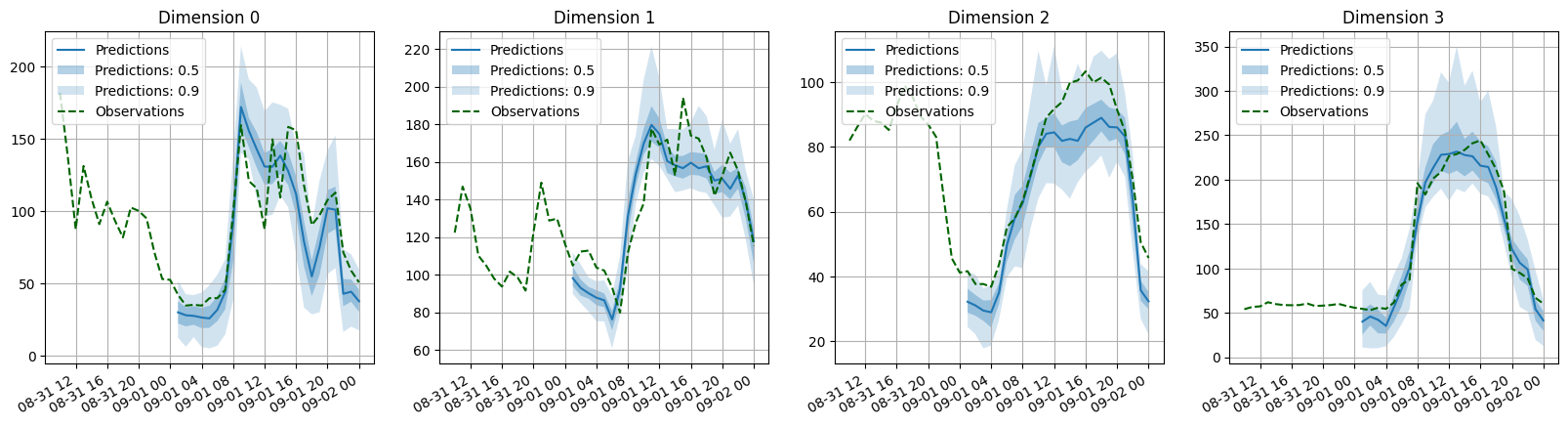}
  % \vskip -0.1in
  \caption{Sample forecasts of VSULD in the electricity dataset} 
\end{figure}

\begin{figure}[H]
  % \centering
  % \vspace{-0.07in}
  \includegraphics[width=\linewidth]{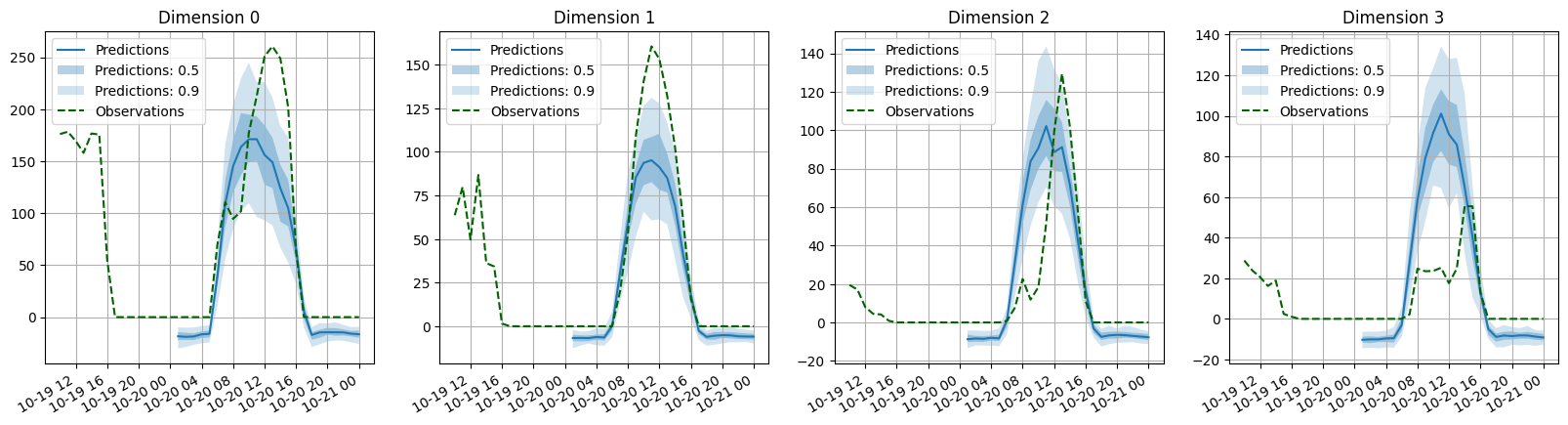}
  % \vskip -0.1in
  \caption{Sample forecasts of VSULD in the solar dataset} 
\end{figure}

\begin{figure}[H]
  % \centering
  % \vspace{-0.07in}
  \includegraphics[width=\linewidth]{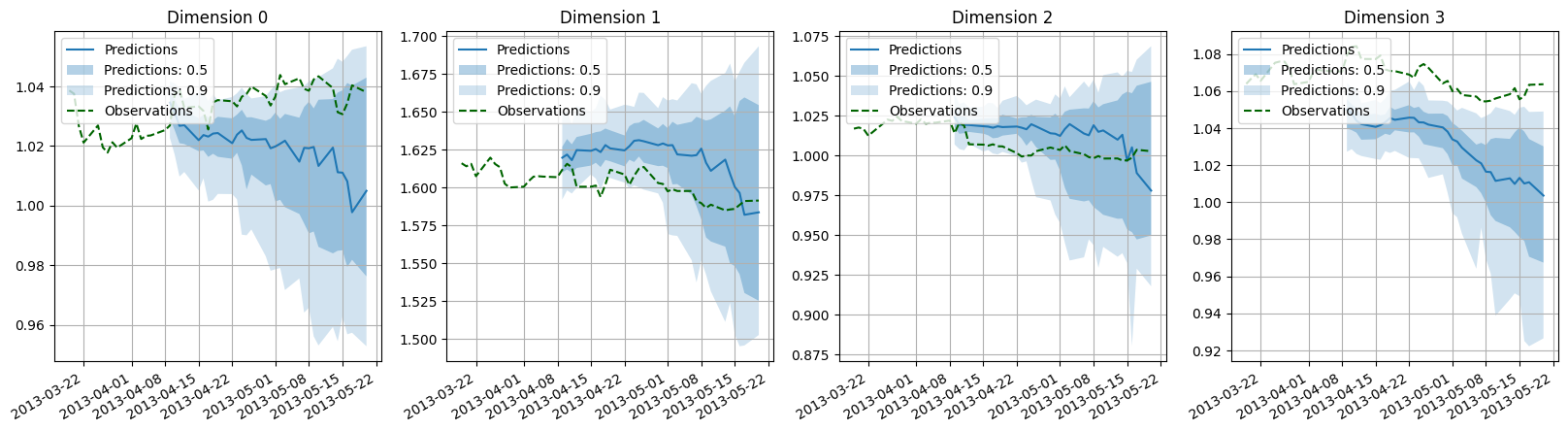}
  % \vskip -0.1in
  \caption{Sample forecasts of VSULD in the exchange rate dataset} 
\end{figure}

\subsection{Forecasts for VSDM}

\begin{figure}[H]
  % \centering
  % \vspace{-0.07in}
  \includegraphics[width=\linewidth]{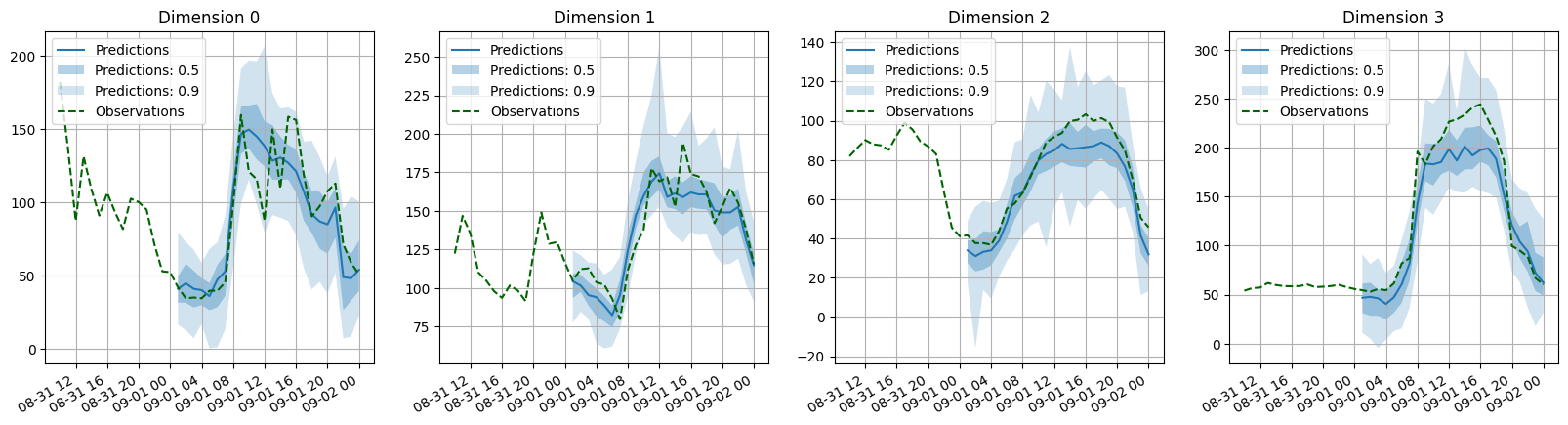}
  % \vskip -0.1in
  \caption{Sample forecasts of VSDM in the electricity dataset} 
\end{figure}

\begin{figure}[H]
  % \centering
  % \vspace{-0.07in}
  \includegraphics[width=\linewidth]{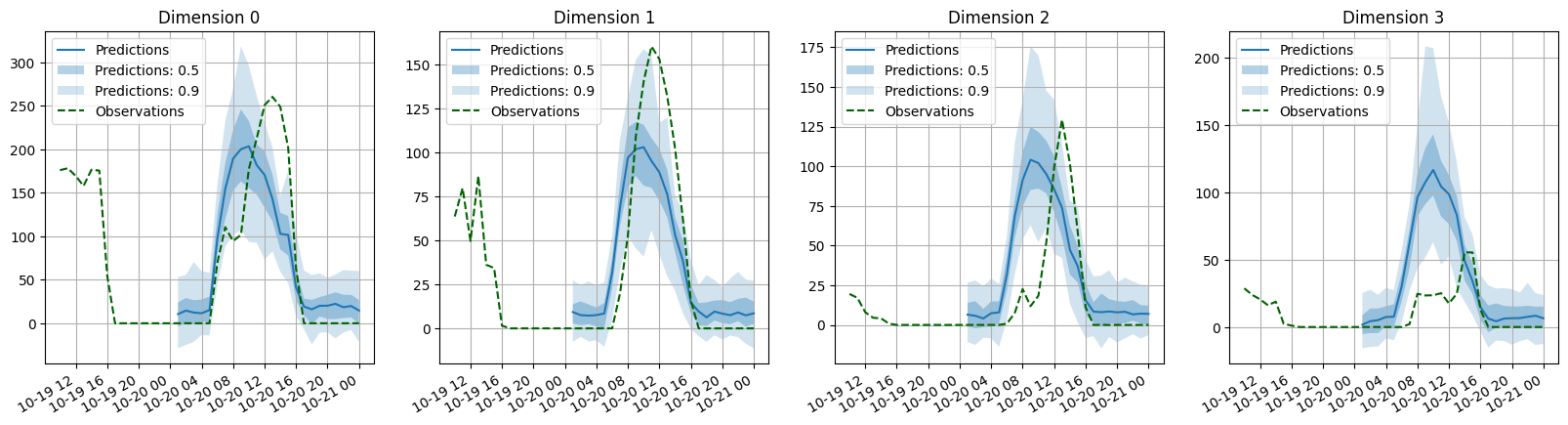}
  % \vskip -0.1in
  \caption{Sample forecasts of VSDM in the solar dataset} 
\end{figure}

\begin{figure}[H]
  % \centering
  % \vspace{-0.07in}
  \includegraphics[width=\linewidth]{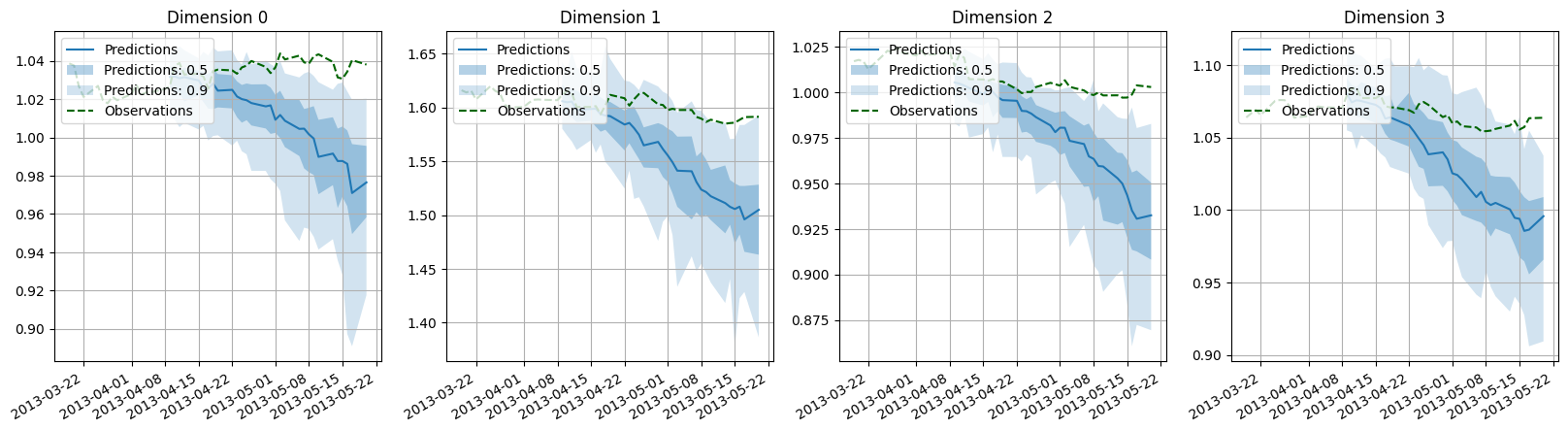}
  % \vskip -0.1in
  \caption{Sample forecasts of VSDM in the exchange rate dataset} 
\end{figure}

\subsection{Forecasts for CLD}

\begin{figure}[H]
  % \centering
  % \vspace{-0.07in}
  \includegraphics[width=\linewidth]{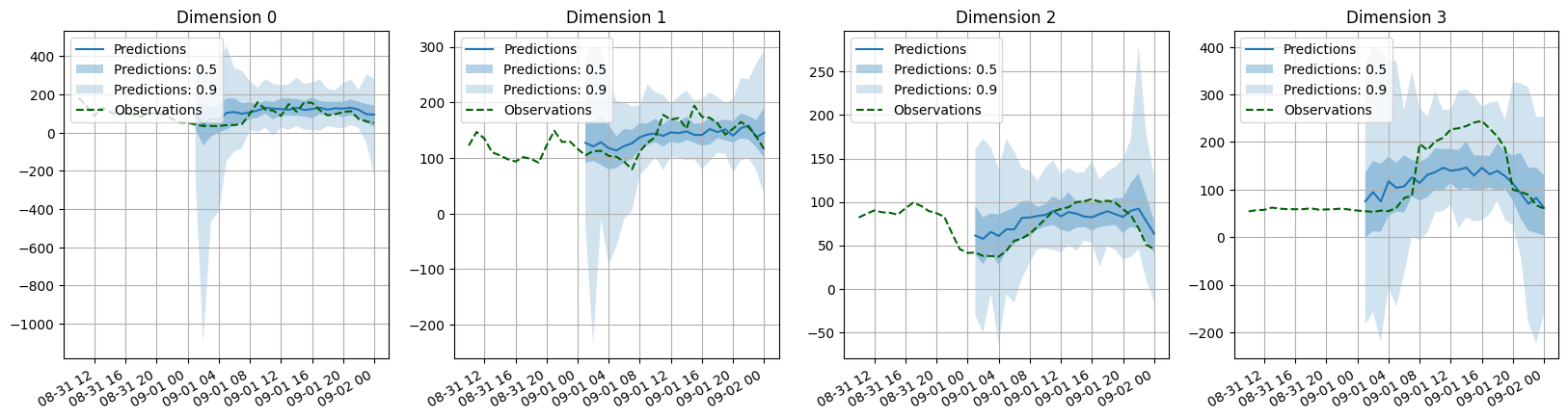}
  % \vskip -0.1in
  \caption{Sample forecasts of CLD in the electricity dataset} 
\end{figure}

\begin{figure}[H]
  % \centering
  % \vspace{-0.07in}
  \includegraphics[width=\linewidth]{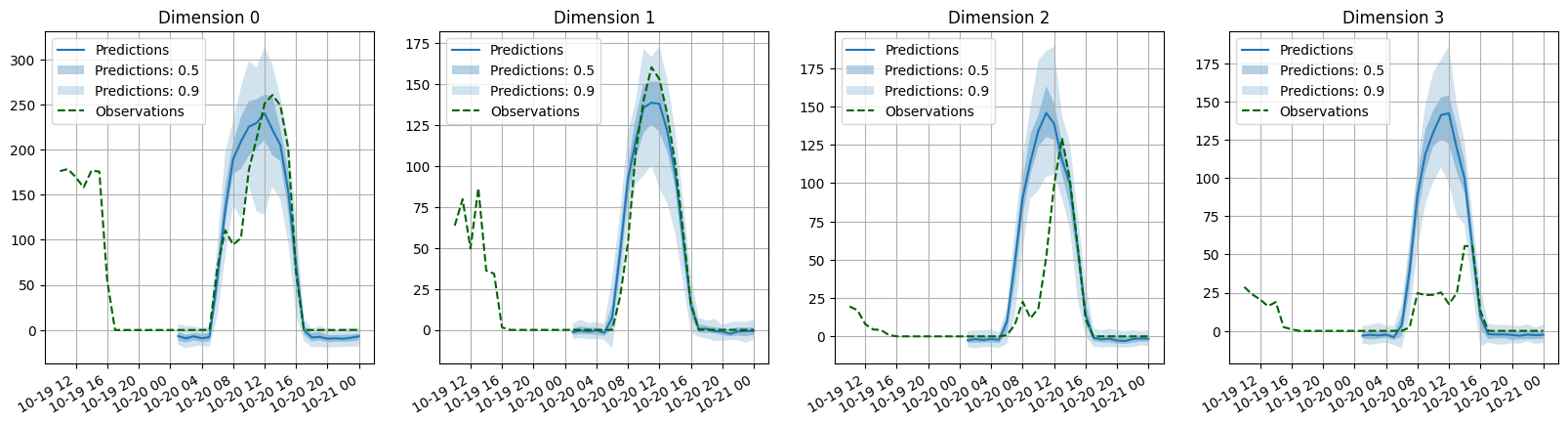}
  % \vskip -0.1in
  \caption{Sample forecasts of CLD in the solar dataset} 
\end{figure}

\begin{figure}[H]
  % \centering
  % \vspace{-0.07in}
  \includegraphics[width=\linewidth]{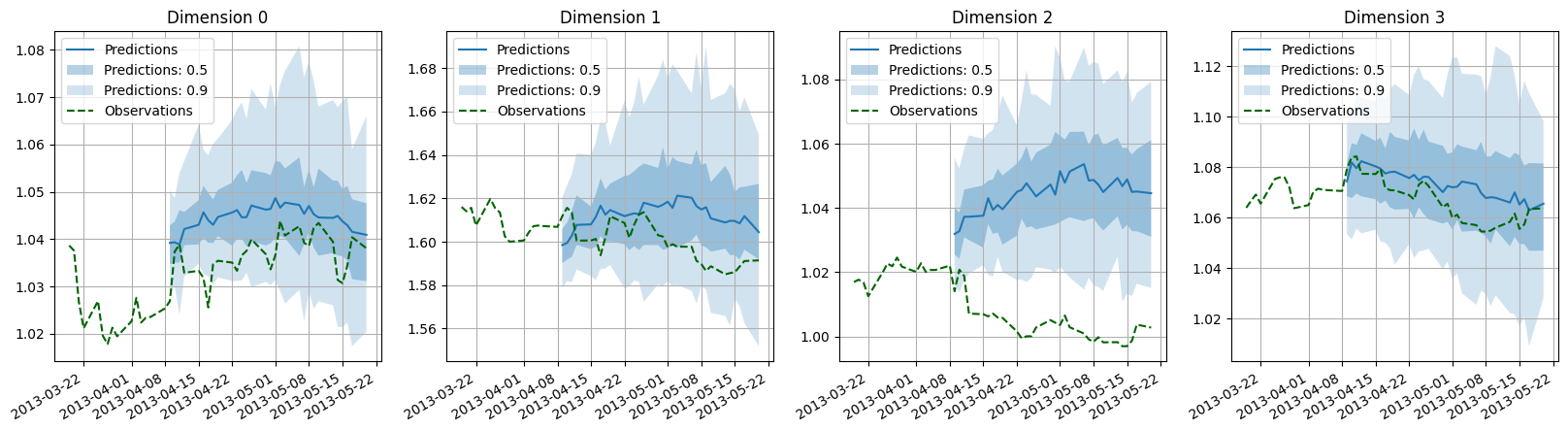}
  % \vskip -0.1in
  \caption{Sample forecasts of CLD in the exchange rate dataset} 
\end{figure}